\documentclass[journal,10pt]{IEEEtran} 

\usepackage{amsmath,amssymb,amsfonts,bm}
\usepackage{amsthm}
\usepackage{mathrsfs}
\usepackage{mathtools}

\usepackage{graphicx}
\usepackage[caption=false,font=footnotesize]{subfig}
\usepackage{array}
\usepackage{color}

\usepackage{cite}

\usepackage[ruled,vlined,linesnumbered,resetcount]{algorithm2e}
\SetKwRepeat{Do}{do}{while}

\usepackage{url}

\newtheorem{theorem}{Theorem}
\newtheorem{corollary}{Corollary}

\begin{document}
\title{Generalization Analysis and Method for Domain Generalization for a Family of Recurrent Neural Networks}

\author{Atefeh Termehchi, Ekram Hossain,  \IEEEmembership{Fellow, IEEE}, and  Isaac Woungang,  \IEEEmembership{Senior Member, IEEE}
%\affil{Department of Electrical and Computer Engineering at the University of Manitoba, Winnipeg, Canada}
%\affil{\textit{Department of Computer Science, Toronto Metropolitan University, Toronto, Canada}}
\thanks{Atefeh Termehchi and Ekram Hossain are with the Department of Electrical and Computer Engineering at the University of Manitoba, Winnipeg, Canada (emails: atefeh.termehchi@umanitoba.ca and ekram.hossain@umanitoba.ca). Isaac Woungang is with the Department of Computer Science, Toronto Metropolitan University, Toronto, Canada (email: iwoungan@torontomu.ca).}}

\maketitle

\begin{abstract}
Deep learning (DL) has driven broad advances across scientific and engineering domains. Despite its success, DL models often exhibit limited interpretability and generalization, which can undermine trust, especially in safety-critical deployments. As a result, there is growing interest in (i) analyzing interpretability and generalization and (ii) developing models that perform robustly under data distributions different from those seen during training (i.e. domain generalization). However, the theoretical analysis of DL remains incomplete. For example, many generalization analyses assume independent samples, which is violated in sequential data with temporal correlations. Motivated by these limitations, this paper proposes a method to analyze interpretability and out-of-domain (OOD) generalization for a family of recurrent neural networks (RNNs). Specifically, the evolution of a trained RNN’s states is modeled as an unknown, discrete-time, nonlinear closed-loop feedback system. Using Koopman operator theory, these nonlinear dynamics are approximated with a linear operator, enabling interpretability. Spectral analysis is then used to quantify the worst-case impact of domain shifts on the generalization error. Building on this analysis, a domain generalization method is proposed that reduces the OOD generalization error and improves the robustness to distribution shifts. Finally, the proposed analysis and domain generalization approach are validated on practical temporal pattern-learning tasks.
\end{abstract}
\section{Introduction}
Over the past decade, deep learning (DL) has driven impressive progress across scientific and engineering domains such as physics, robotics, and communications, to name a few. DL advances these areas by providing data-driven solutions when analytical modeling or optimization is difficult to formulate, relies on oversimplified assumptions, or becomes computationally intractable. However, the concerns about DL models have increasingly come to the forefront because they are frequently regarded as ``non-transparent" \cite{suh2025survey, csahin2025unlocking, xu2025interpretability}. In particular, they are viewed as non-transparent since, even with full access to the parameters (weights and biases) and activation functions, it remains difficult to determine what knowledge or decision rules the DL model has internalized. This is largely due to the repeated application of nonlinear transformations within the model. This opacity brings two connected challenges: limited interpretability and limited generalization beyond the training data. 

Interpretability refers to the extent to which a human can understand the reasons behind a model’s predictions, i.e. what information the model uses and how it maps the inputs to the outputs. Although the literature offers many perspectives, a widely cited definition describes interpretability as the ability to provide explanations in human-understandable terms \cite{ zhang2021survey}. On the other hand, generalization refers to a model’s ability to maintain a reliable performance on unseen data beyond the training set. It is often discussed in terms of in-distribution generalization, where the target data follow (approximately) the training distribution, and out-of-distribution (OOD) generalization, where the target data come from different distributions. Moreover, as a closely related concept, domain generalization (DG) aims to learn models whose performance remains robust under such distribution shifts. In contrast to other learning paradigms such as domain adaptation or transfer learning, DG assumes that the target-domain data is unavailable during training \cite{zhou2022domain}.

The limited interpretability and generalization of DL models undermine trust and impede their adoption in high-stakes settings such as healthcare, smart power systems, and security-critical infrastructure. In this context, model failures can lead to serious safety risks and substantial financial losses \cite{zhang2021survey}. Therefore, many research efforts focus on analyzing interpretability and generalization of DL models, as well as developing DG methods \cite{xu2025interpretability, hellstrom2025generalization, zhou2022domain}. However, the field is still not mature. For interpretability, some methods face practical limitations (e.g., rule- or logic-based explanations), and many approaches target only one facet of transparency, either by explaining input-output attributions (e.g., perturbation-based tests) or by analyzing internal mechanisms (e.g., information-bottleneck views).
 For generalization analysis, major gaps exist in settings with dependent data (e.g., sequential data), where temporal correlations and anytime-valid guarantees are required. Although some progress has been made, the area is in its early development, and many challenges remain to be addressed \cite{hellstrom2025generalization,rodriguez2024information}. %still faces many unresolved challenges

Sequential data is ubiquitous and underpins decision-making in many sectors, including healthcare and energy systems. For modeling such sequences and supporting decisions in tasks with temporal patterns, recurrent neural networks (RNNs) have long been a core architecture. They have been widely used in applications such as power-load forecasting, speech recognition, and video/temporal vision tasks, since they provide a simple and efficient way to process time series. 
Consequently, this paper introduces an interpretable representation of RNNs (with a focus on Long Short-Term Memory [LSTM]) that captures both input--output attributions and internal mechanisms. Specifically, we model the dynamics of the LSTM internal states and their contributions to the output as an unknown, discrete-time, nonlinear closed-loop feedback system. To analyze these unknown dynamics, we leverage the Koopman operator theory and its data-driven identification capability. We approximate the Koopman spectral characteristics using the dynamic mode decomposition with control (DMDc). Domain shifts are modeled as additive disturbances on the inputs. We then use the $\text{H}_\infty$ norm to quantify the worst-case impact of such disturbances on the trained LSTM and to derive generalization-error characterizations under both covariate shift and concept drift. 
Building on these analyses, we propose a DG method that reduces the OOD generalization error of a trained LSTM under the input distribution shifts. Specifically, we formulate an optimization problem with linear matrix inequality (LMI) constraints that minimizes the effect of domain shift on the LSTM state dynamics, thereby improving the robustness to distribution changes.
%The primary objective of this paper is to introduce an analytical method for evaluating the OOD generalization of DRL algorithms. We begin by modeling the evolution of states and actions in trained DRL algorithms as unknown, discrete, stochastic, and nonlinear dynamical functions. Domain shifts are then represented by incorporating an additive disturbance vector into the conditional transition probability function. To analyze the unknown dynamical functions, we employ Koopman operator theory, leveraging its data-driven identification capabilities. To approximate the spectral features of the Koopman operator, we apply both dynamic mode decomposition (DMD) and exact DMD. Subsequently, we use the H$_\infty$ norm to assess these spectral characteristics and quantify the worst-case impact of domain shifts on the trained DRL model.

\subsection{Motivation and Prior Work}
Interpretability analyses seek to make the decision-making behavior of DL models more transparent by moving beyond accuracy-focused evaluation toward understanding how and why decisions or predictions are produced. Although a substantial body of work has proposed interpretability analysis methods for DL models (see \cite{xu2025interpretability,zhang2021survey}, and the references therein), the field remains immature. This immaturity is reflected in several limitations. First, approaches that produce human-readable explanations (e.g., rule- or decision-tree-based methods such as  \cite{aguilar2022towards}) can become impractical when the explanation complexity is high; thus, complexity must be carefully controlled (e.g., by limiting decision-tree depth). Second, many methods provide only local explanations for a target input (e.g., based on feature values, saliency scores, or gradients), offering limited global insight into the model's overall behavior. For example, LIME \cite{ribeiro2016should} provides local interpretability. Although Global--LIME \cite{li2023g} improves the stability and consistency of LIME’s local explanations, it still mainly explains input--output attributions and does not capture internal mechanisms. Third, many methods target only one facet of transparency, focusing either on input--output relationships or internal mechanisms (e.g., mechanistic interpretability \cite{bereska2024mechanistic} or information-bottleneck-based views \cite{he2025information}).
%Although a substantial body of work has proposed methods for interpretability analysis of DL models (see \cite{xu2025interpretability,zhang2021survey}, and the references therein), the field remains immature. Many approaches provide local explanations that are specific to a target input (e.g., based on feature values, saliency scores, or gradients), but do not yield global insight into the model’s overall behavior. Some methods aim to produce inherently human-readable explanations, such as rule-based or logic-based descriptions; however, their usefulness depends on controlling explanation complexity (e.g., limiting the depth of a decision tree), since overly complex explanations are difficult to interpret in practice. Moreover, existing methods often focus on only one aspect of transparency, either explaining the input--output relationship or analyzing internal mechanisms (e.g., mechanistic interpretability or information-bottleneck-based views). %Despite extensive study, interpretability remains an umbrella term without a universally accepted definition. A commonly cited view describes interpretability as the ability to generate explanations that are meaningful and understandable to humans \cite{doshi2017towards, zhang2021survey}. 
%While many methods exist for in-distribution generalization analysis (with the assumption of independent and identically distributed [i.i.d.]), developing generalization analysis under distribution shift still needs further advances \cite{liu2021towards}.  

In addition, generalization analysis studies how a model’s performance changes on unseen data. It is closely related to interpretability, since understanding how and why a model makes decisions helps assess when that model will remain reliable.
 Generalization analysis is commonly divided into in-distribution (unseen data follows roughly the training distribution) and out-of-distribution (unseen data comes from a different distribution). While in-distribution generalization analysis (under the independent and identically distributed [i.i.d.] assumption) still faces open challenges, generalization under distribution shift is even harder to analyze \cite{liu2021towards}. In particular, estimating the generalization error (GE) under domain shift is more challenging for several reasons. First, the target (unseen) data distribution is typically unknown and may deviate from the training distribution in arbitrary ways. Second, different shift types (such as covariate, label, and concept shift) can arise and often require distinct analytical treatments. Third, there is an inherent trade-off between adopting assumptions broad enough to capture realistic shifts and deriving the GE bounds that are tight enough to be practically informative. For instance, Weber et al.~\cite{weber2022certifying} analyze OOD generalization using a quantum-chemistry-inspired formulation and the Hellinger distance. However, their GE bound is valid only when the Hellinger distance between the training and shifted distributions remains below a threshold, limiting its applicability to small domain shifts.
 %For instance, Weber et al.~\cite{weber2022certifying} propose an OOD generalization analysis approach based on a Gram matrix theory from quantumchemistry and using the Hellinger distance. A key limitation of this approach is its strict validity condition: the bound is only applicable when the Hellinger distance remains below a prescribed threshold, which effectively restricts its use to small domain shifts.

 Moreover, generalization analysis in settings with dependent data (whether under the same distribution or under a distribution shift) remains challenging due to the temporal correlations. Although some progress has been made, the field is still in its early stages \cite{hellstrom2025generalization,rodriguez2024information}. One setting in which the sequential training data are used is deep reinforcement learning (DRL). In \cite{termehchi2025koopman}, the authors use a Koopman-based interpretability model for DRL to evaluate the generalization under covariate shift. Koopman operator theory and DMD have recently gained attraction as tools for characterizing the dynamical behavior of internal states and parameters in DL-based models \cite{dogra2020optimizing,rozwood2024koopman}. These data-driven dynamical-systems techniques provide a principled way to probe otherwise non-transparent learning dynamics. For example, \cite{rozwood2024koopman} shows that a Koopman-based representation can model the expected time evolution of a DRL value function, which can be leveraged to improve the value estimation and, in turn, enhance the DRL performance. The behavior of the RNN family can be modeled as a dynamical system because it updates an internal state over time as it processes a sequence. This recurrent state update captures the temporal dependencies via fixed update rules, analogous to a dynamical system. In this paper, we leverage the Koopman operator theory to analyze interpretability and generalization in RNNs, with a focus on LSTMs. The proposed method can also be extended to other RNN architectures.

\subsection{Contributions}
The main contributions of this paper are summarized below:
\begin{itemize}
    \item \textbf{Interpretability analysis}: Focusing on LSTMs, we model the state dynamics and the input--state--output relationship of a trained LSTM as an unknown, discrete-time, nonlinear closed-loop feedback system. We then apply the Koopman theory and DMDc to obtain a linear approximation of this dynamical system, enabling a global analysis of both input-output attributions and internal mechanisms of the LSTM. 
    \item \textbf{Generalization analysis}: We perform a spectral analysis on the proposed interpretable model to derive a bound on the GE. We thereby quantify the worst-case impact of domain shifts (i.e. covariate shift and concept shift) on the LSTM states and the resulting outputs.  
    \item \textbf{DG method}: We propose a DG method that mitigates the performance degradation induced by the domain shift. Specifically, we modify the trained LSTM architecture to mitigate the effect of domain shift (without retraining the model), thereby improving the robustness to OOD data.
\end{itemize}

\section{Background Concepts and Formal Definitions}
This section presents the core background theory and mathematical tools used in this work.
\subsection{Definition of Domain Shift in LSTM}
%\red{\cite{sutskever2014sequence}:Let $\mathbf{x}_{1:T}=(x_1,\dots,x_T)$ denote the input sequence and $\mathbf{y}_{1:T'}=(y_1,\dots,y_{T'})$ its corresponding output sequence, where $T'$ may differ from $T$. The LSTM's objective is to model the conditional distribution
%\[
%p\!\left(\mathbf{y}_{1:T'} \mid \mathbf{x}_{1:T}\right).
%\]}
Let $\mathbf{x}_{0:T} = (x_0, \dots, x_T)$ denotes the input sequence of length $T$ and 
$\mathbf{y}_{0:T'} = (y_0, \dots, y_{T'})$ denotes the corresponding output sequence of length $T'$, where in general $T' \neq T$ as in sequence-to-sequence tasks~\cite{sutskever2014sequence}.
Assume that the LSTM is trained on pairs of sequences 
$(\mathbf{x}_{0:T}, \mathbf{y}_{0:T'}) \sim \mathcal{P}_i,$ and it is evaluated on test dataset $(\mathbf{x}_{0:T}, \mathbf{y}_{0:T'}) \sim \mathcal{P}_j$ (hereafter, the terms test and target are used interchangeably). A domain shift occurs when the joint distribution of the input--output sequences differs between training and test:
\[
\mathcal{P}_i(\mathbf{x}_{0:T}, \mathbf{y}_{0:T^{'}}) \;\neq\; 
\mathcal{P}_j(\mathbf{x}_{0:T}, \mathbf{y}_{0:T^{'}}).
\]
%Let \( (\mathbf{x}_t, \mathbf{y}_t) \) denotes the input-output pair at time step \( t \), where $\mathbf{x}_t\in \mathbb{R}^{n_{\text{in}}}$ and $\mathbf{y}_t\in \mathbb{R}^{n_{\text{out}}}$. Assume the LSTM is trained on sequences \( \{ (\mathbf{x}_t, \mathbf{y}_t) \}_{t=0}^T \sim \mathcal{P}_{i} \), and at test time, it is evaluated on sequences \( \{ ({\mathbf{x}}_t, {\mathbf{y}}_t) \}_{t=0}^T \sim \mathcal{P}_{j} \). Then, a domain change occurs when the joint distribution of the input-output sequences differs between training and testing:
%\[
%\mathcal{P}_{i}(\mathbf{x}_{0:T}, \mathbf{y}_{0:T}) \neq \mathcal{P}_{j}({\mathbf{x}}_{0:T}, {\mathbf{y}}_{0:T}).
%\]
From a causal learning perspective and based on Bayes' rule
\begin{align*}
 \mathcal{P}_{i}(\mathbf{x}_{0:T}, \mathbf{y}_{0:T^{'}}) &= 
 \mathcal{P}^{y|x}_i(\mathbf{y}_{0:T^{'}} \mid \mathbf{x}_{0:T}) \times \mathcal{P}^{x}_i(\mathbf{x}_{0:T})\\
 &= \mathcal{P}^{x|y}_i(\mathbf{x}_{0:T} \mid \mathbf{y}_{0:T^{'}}) \times \mathcal{P}^{y}_i(\mathbf{y}_{0:T^{'}}),% \\
 %\mathcal{P}_{j}(\mathbf{x}_{0:T}, \mathbf{y}_{0:T^{'}}) &= \mathcal{P}^{y|x}_j(\mathbf{y}_{0:T^{'}} \mid \mathbf{x}_{0:T}) \times \mathcal{P}^{x}_j(\mathbf{x}_{0:T}) \\
 %&=\mathcal{P}^{x|y}_j(\mathbf{x}_{0:T} \mid \mathbf{y}_{0:T^{'}}) \times \mathcal{P}^{y}_j(\mathbf{y}_{0:T^{'}}).
\end{align*}
the domain shift is classified into two general classes:
\begin{itemize}
    \item \textbf{Covariate Shift (Input Shift)}: The situation in which the training input points and test input points follow different distributions, while the conditional distribution of the output given the input remains unchanged~\cite{sugiyama2007covariate}:
    \begin{align*}
    \mathcal{P}^{x}_i(\mathbf{x}_{0:T}) &\neq \mathcal{P}^{x}_j(\mathbf{x}_{0:T}), \\
    \mathcal{P}^{y|x}_i(\mathbf{y}_{0:T^{'}} \mid \mathbf{x}_{0:T}) &= \mathcal{P}^{y|x}_j(\mathbf{y}_{0:T^{'}} \mid \mathbf{x}_{0:T}).
    \end{align*}
    \item \textbf{Concept Drift}: The situation in which the conditional distribution of the output variable given the input features changes, even when the input distribution remains the same~\cite{gama2014survey}:
    \begin{align*}
    \mathcal{P}^{y|x}_i(\mathbf{y}_{0:T^{'}} \mid \mathbf{x}_{0:T}) &\neq \mathcal{P}^{y|x}_j(\mathbf{y}_{0:T^{'}} \mid \mathbf{x}_{0:T}).
    \end{align*}
%This conditional distribution change may result from:
  %  \begin{itemize}
  %      \item a change in the prior class probabilities: \( \mathcal{P}^{y}_i(\mathbf{y}_{0:T}) \neq \mathcal{P}^{y}_j(\mathbf{y}_{0:T}) \),
 %       \item a shift in the conditional distribution of inputs given the output: \( \mathcal{P}^{x|y}_i(\mathbf{x}_{0:T} \mid \mathbf{y}_{0:T}) \neq \mathcal{P}^{x|y}_j(\mathbf{x}_{0:T} \mid \mathbf{y}_{0:T}) \).
%    \end{itemize}
\end{itemize}
In this paper, we focus on domain shifts that involve changes in the input distribution. If the conditional distribution of the output given the input remains unchanged across domains, the shift is \textbf{covariate shift}; and if the conditional distribution changes, it is \textbf{concept drift}. 
%In addition, we assume that the output distribution changes as a result of a shift in the input distribution. Cases where the output distribution remains unaffected by the covariate shift are considered trivial.
%\noindent \textbf{Note 1}.
We let $\mathcal{D}_i$ denote the training domain associated with a joint distribution 
$\mathcal{P}_i(\mathbf{x}_{0:T}, \mathbf{y}_{0:T'})$. 
Similarly, $\mathcal{D}_j$ denotes the test domain with distribution 
$\mathcal{P}_j(\mathbf{x}_{0:T}, \mathbf{y}_{0:T'})$.

\subsection{Definition of Generalization Error}
The generalization error (GE) for regression tasks is defined as follows. %Let \( \mathcal{L} \) denote a general loss function. 
%Considering \( D_i \subset \mathcal{D} \) to be a dataset consisting of sequences of the form$\left( (\mathbf{x}_{0}, \mathbf{y}_{0}), (\mathbf{x}_{1}, \mathbf{y}_{1}), \ldots, (\mathbf{x}_{T}, \mathbf{y}_{T}) \right) \in \mathcal{D}_i$ where \( \mathbf{y}_{t} \) is the corresponding response to \( \mathbf{x}_{t} \) at time step \( t \). Each sequence \( \left((\mathbf{x}_{0}, \mathbf{y}_{0}), \ldots, (\mathbf{x}_{T}, \mathbf{y}_{T})\right) \in D_i \) is assumed to be drawn from an underlying data-generating distribution \( \mathcal{P}_i \). %In addition, the distribution over $(\mathbf{x}_{0},..., \mathbf{x}_{T})$ follows \( \mathcal{P}^x_i \).
%\textbf{Regression}:
Let $\mathcal{L}$ be a general loss function. 
Suppose that the training dataset $S_i = \{ (\mathbf{x}_{0:T}^{(n)}, \mathbf{y}_{0:T'}^{(n)}) \}_{n=1}^m$ is drawn from the domain distribution $\mathcal{D}_i$.
The \textit{empirical loss} of a regression model $M$ on $S_i$ is defined as: 
\(
\mathcal{L}_{\text{emp}}(M, S_i) = \frac{1}{m}\sum_{n=1}^m 
\mathcal{L}\!\left(\mathbf{y}_{0:T'}^{(n)}, M(\mathbf{x}_{0:T}^{(n)})\right),
\)
where \( M \) denotes the model trained on \( S_i \). The corresponding \textit{expected loss} of $M$ on a test domain $\mathcal{D}_j$ is:
\(
\mathcal{L}_{\text{exp}}(M, \mathcal{D}_j) = 
\mathbb{E}_{(\mathbf{x}_{0:T}, \mathbf{y}_{0:T'}) \sim \mathcal{D}_j}
\left[ \mathcal{L}(\mathbf{y}_{0:T'}, M(\mathbf{x}_{0:T})) \right],
\) 
which represents the model's average performance over a dataset $\{(\mathbf{x}_{0:T}, \mathbf{y}_{0:T'}) \sim \mathcal{D}_j \}$. The \textit{generalization error} of $M$ is the absolute difference between these two quantities~\cite{he2025information}:
\[
\text{GE}(M, S_i, \mathcal{D}_j) =
\left| \mathcal{L}_{\text{emp}}(M, S_i) - \mathcal{L}_{\text{exp}}(M, \mathcal{D}_j) \right|.
\]

\subsection{Koopman Operator}
The Koopman operator allows a nonlinear dynamical system, originally defined in a finite-dimensional space, to be represented as a linear system in an infinite-dimensional function space. This reformulation makes it possible to apply classical spectral analysis techniques to study complex behaviors that emerge from nonlinear dynamics \cite{proctor2018generalizing}. Consider a nonlinear discrete-time autonomous system described by the state-update equation:
\[
\mathbf{s}_{t} = f(\mathbf{s}_{t-1}),
\]
where $\mathbf{s}_t \in \mathbb{R}^{n_{\text{state}}}$ is the state vector, function $f: \mathbb{R}^{n_{\text{state}}} \to \mathbb{R}^{n_{\text{state}}}$ describes the system's evolution, and $t \in \mathbb{Z}$ denotes the time index. We now introduce a collection of functions $\phi: \mathbb{R}^{n_{\text{state}}} \to \mathbb{R}$, referred to as observables. This collection of observables forms an infinite-dimensional Hilbert space, denoted by~$\mathcal{O}$. %A common choice for $\mathcal{\phi}$ is the space of square-integrable functions with respect to a Lebesgue measure. 
The Koopman operator $\mathcal{K}$ is defined as a linear operator acting on these observables:
\[
\mathcal{K}\phi(\mathbf{s}_{t-1}) := \phi(f(\mathbf{s}_{t-1})).
\]
This definition implies that the Koopman operator advances the evaluation of observables according to the system dynamics. Although the underlying system may be nonlinear, the Koopman operator itself is linear. However, it generally operates in an infinite-dimensional functional space. %The Koopman operator captures the full behavior of the dynamical system as long as the space of observables $\mathcal{O}$ is sufficiently rich. In particular, if each coordinate function
The Koopman operator captures the complete behavior of the dynamical system, provided that the space of observables \( \mathcal{O} \) is sufficiently rich. In particular, if each coordinate 
function\footnote{A function that takes full vector $
\mathbf{s} =
\begin{bmatrix}
s^1 \\
s^2 \\
\vdots \\
s^{n_{\text{state}}}
\end{bmatrix}
$ and returns $i\text{th}$ component} \( \mathbf{s} \mapsto s^i \), for \( i = 1, \ldots, n_{\text{state}} \), belongs to \( \mathcal{O} \), then the observable space is considered rich. Under this condition, the Koopman framework is capable of fully representing the system's evolution. %$\mathbf{s} \mapsto s^i$ for $i = 1, \ldots, {n_{\text{state}}}$ belongs to $\mathcal{O}$, then the space of observables is called rich and accordingly the Koopman framework retains all information about the system's evolution \cite{korda2018linear}.

\textbf{Extended Koopman Operator for Non-autonomous Systems}: Koopman operator theory was originally developed for autonomous dynamical systems (i.e. systems without exogenous inputs or forcing). However, several approaches have been proposed to extend the Koopman framework to non-autonomous systems \cite{williams2016extending,proctor2018generalizing,korda2018linear}. In this work, we adopt the method presented in \cite{korda2018linear}, summarized as follows. Consider a discrete-time non-autonomous system described by the state-update equation:
\begin{equation}
\mathbf{s}_{t} = f(\mathbf{s}_{t-1},\mathbf{x}_{t}),
\label{state-update}
\end{equation}
where $\mathbf{x}_t \in \mathbb{R}^{n_{\text{in}}}$ is input vector at time $t \in \mathbb{Z}$. In \cite{korda2018linear}, the Koopman operator is defined for the non-autonomous system by extending the state space.
The extended state space is defined as \( \mathbb{R}^{n_{\text{in}}} \times \mathcal{X}^\infty \), where \( \mathcal{X}^\infty \) is the space of infinite input sequences \(\boldsymbol{x}_t = (\mathbf{x}_t, \mathbf{x}_{t+1}, \dots) \). %with \( \mathcal{x}_i \in \mathcal{X}^\infty  \).
Accordingly, the extended state is:
\begin{equation}
\zeta_{t} = 
\begin{bmatrix}
\mathbf{s}_{t-1} \\
\boldsymbol{x}_t
\end{bmatrix},
\label{extend_state}
\end{equation}
and the dynamics of the extended state is described as:
\begin{equation}
 \zeta_{t+1} = {F}(\zeta_{t}) := 
\begin{bmatrix}
f(\mathbf{s}_{t-1},\boldsymbol{x}_t(0)) \\
\ell(\boldsymbol{x}_t)
\end{bmatrix},    
\label{extend_dynamics}
\end{equation}
where $\boldsymbol{x}_t(0)$ denotes first element of input sequence in $\boldsymbol{x}_t$ and $\ell(\boldsymbol{x}_t)= \boldsymbol{x}_{t+1}$. The Koopman operator \( \mathcal{K} : \mathcal{O} \to \mathcal{O} \), associated with the system in (\ref{extend_dynamics}) is defined as:
\begin{equation}
\mathcal{K} \varphi(\zeta_{t}) = \varphi(F(\zeta_t)),
\label{extend_koopman}
\end{equation}
where \( \zeta_t \in \mathbb{R}^n \times \mathcal{X}^\infty \) is the extended state and \( \varphi : \mathbb{R}^{n_{\text{in}}} \times \mathcal{X}^\infty \to \mathbb{R} \) is an element of the observable space \( \mathcal{O} \), which is an infinite-dimensional Hilbert space.

\subsection{Dynamic Mode Decomposition with Control}
To find a finite dimensional approximation to the operator $\mathcal{K}$ in (\ref{extend_koopman}), we adopt the dynamic mode decomposition with control (DMDc) proposed in \cite{proctor2016dynamic}. DMDc is a data-driven algorithm assuming access to a dataset \( \{ (\zeta_t, \zeta_{t+1}) \}_{t=1}^{\bar{T}} \), where each pair satisfies \( \zeta_{t+1} = F(\zeta_t) \). It then seeks a matrix \( \hat{\mathbf{A}} \), representing the finite-dimensional approximation of the Koopman operator \( \mathcal{K} \), by minimizing the following least-squares objective:
\begin{equation}
\sum_{t=1}^{\bar{T}} \left\| \boldsymbol{\varphi}(\zeta_{t+1}) - \hat{\mathbf{A}} \boldsymbol{\varphi}(\zeta_{t}) \right\|_2^2,
\label{obj_edmd}
\end{equation}
where \( \boldsymbol{\varphi}(\zeta_t) = [\varphi_1(\zeta), \dots, \varphi_{N_\varphi}(\zeta)]^\top \) is a vector of observables, with each \( \varphi_i : \mathbb{R}^n \times \ell(\boldsymbol{x}_t) \to \mathbb{R} \) for \( i = 1, \dots, N_\varphi \).
The dimensional of the extended state \( \zeta_t \) is typically infinite due to the input sequence \(\ell(\boldsymbol{x}_t) \). Therefore, the objective can only be evaluated in practice if the observables \( \varphi_i \) are selected such that the problem becomes finite-dimensional. To address this, and without loss of generality (due to linearity and causality), the observable vector \( \boldsymbol{\varphi} \) is defined as in~\cite{korda2018linear}:
\begin{equation}
    \boldsymbol{\varphi}(\zeta) = \begin{bmatrix}
        \psi(\mathbf{s}_{t-1}) \\
        \boldsymbol{x}_t(0)
    \end{bmatrix},
\end{equation}
where \( \psi = [\psi_1, \dots, \psi_N] \), and each \( \psi_i: \mathbb{R}^{n_{\text{state}}} \to \mathbb{R}\) is generally nonlinear. The functions \( \psi_i \) are chosen so that the space of observables is rich enough to capture the system’s dynamics. Since we are not interested in the dynamics that governs the future values of the input sequence, we can ignore the last \( m \) components in each term \( \boldsymbol{\varphi}(\zeta_{t+1}) - \hat{\mathbf{A}} \boldsymbol{\varphi}(\zeta_{t}) \) in equation (\ref{obj_edmd}). Let \( \tilde{\mathbf{A}} \) represent the first \( N \) rows of matrix \( \mathbf{A} \), and let decompose it as: 
\(
\tilde{\mathbf{A}} = [\mathbf{A} \;\; \mathbf{B}],
\)
where \( \mathbf{A} \in \mathbb{R}^{N \times N} \) and \( \mathbf{B} \in \mathbb{R}^{N \times n_{\text{in}}} \). Then, the optimization objective of (\ref{obj_edmd}) becomes:
\begin{equation}
\min_{\mathbf{A}, \mathbf{B}} \sum_{t=1}^{\bar{T}} \left\| \boldsymbol{\psi}(\mathbf{s}_{t}) - \mathbf{A} \boldsymbol{\psi}(\mathbf{s}_{t-1}) - \mathbf{B} \boldsymbol{x}_t(0) \right\|_2^2.
\label{obj_EDMD_}
\end{equation}
By solving the optimization problem in (\ref{obj_EDMD_}) and obtaining the matrices \( \mathbf{A} \) and \( \mathbf{B} \), we can approximate the nonlinear state-update equation (\ref{state-update}) with the following linear form:
\begin{equation}
    \boldsymbol{\psi}(\mathbf{s}_{t}) = \mathbf{A}\boldsymbol{\psi}(\mathbf{s}_{t-1}) + \mathbf{B} \mathbf{x}_{t}.
\end{equation}
Several studies have proposed suitable choices for the functions \( \psi_i \), such as the sparse identification of nonlinear dynamical systems method \cite{brunton2016sparse}. A simple and common choice is $
\psi_i(\mathbf{s}) = s^i \quad \text{for all } i \in \{1, 2, \ldots, n_{\text{state}} \},
$ which satisfies the requirement that the space of observables is sufficiently rich.
%\subsubsection{Numerical procedure of calculating $\mathbf{A}$ and $\mathbf{B}$}

The numerical procedure of calculating $\mathbf{A}$ and $\mathbf{B}$ is as follows. We assume access to a dataset:
\[
\mathbf{S} = [\mathbf{s}_{0}, \dots, \mathbf{s}_{\bar{T}-1}],  \mathbf{\Xi} = [\mathbf{s}_{1}, \dots, \mathbf{s}_{\bar{T}}],  \mathbf{X} = [\mathbf{x}_{1}, \dots, \mathbf{x}_{\bar{T}}],
\]
where each data point satisfies the relation \(\mathbf{s}_{t} = f(\mathbf{s}_{t-1},\mathbf{x}_{t})\). Importantly, the data samples are not required to come from a single trajectory of the system described by equation (\ref{state-update})\footnote{This is relevant here: we run the LSTM on many independent sequences, producing samples from multiple trajectories.}. Given this data, we seek for matrices \( \mathbf{A} \in \mathbb{R}^{N \times N} \) and \( \mathbf{B} \in \mathbb{R}^{N \times n_{\text{in}}} \) that best predict the next state in the observable space using a linear model. These matrices are computed by solving the following least-squares optimization problem:
\begin{equation}
\min_{\mathbf{A}, \mathbf{B}} \left\| \boldsymbol{\psi}(\mathbf{\Xi}) - \mathbf{A} \boldsymbol{\psi}(\mathbf{S}) - \mathbf{B}\mathbf{X} \right\|_2,
\label{opt_}
\end{equation}
where $\boldsymbol{\psi}(\mathbf{S}) = [\boldsymbol{\psi}(\mathbf{s}_{0}), \dots, \boldsymbol{\psi}(\mathbf{s}_{\bar{T}-1})]$, $\boldsymbol{\psi}(\mathbf{\Xi}) = [\boldsymbol{\psi}(\mathbf{s}_{1}), \dots, \boldsymbol{\psi}(\mathbf{s}_{\bar{T}})]$, and
\( \boldsymbol{\psi} : \mathbb{R}^{n_{\text{state}}} \to \mathbb{R}^N \) is defined as: 
\(
\boldsymbol{\psi}(\mathbf{s}) =
\begin{bmatrix}
\psi_1(\mathbf{s}) \\
\vdots \\
\psi_N(\mathbf{s})
\end{bmatrix}.\)
As noted earlier, a simple and common choice is 
$\psi_i(\mathbf{s}) = s^i \quad \text{for all } i \in \{1, 2, \ldots, n_{\text{state}} \}$.

The closed-form solution to the optimization problem in (\ref{opt_}) is given by:
\(
[\mathbf{A} \;\; \mathbf{B}] = \boldsymbol{\psi}(\mathbf{\Xi}) [\boldsymbol{\psi}(\mathbf{S}) \;\; \mathbf{X}]^{\dagger},
\)
where \( \dagger \) denotes the Moore–Penrose pseudoinverse. We now specialize the numerical procedure of $\mathbf{A}$ and $\mathbf{B}$ when the identity observable (unity lift), i.e.
\(
\boldsymbol{\psi}(\mathbf{s}) = \mathbf{s} \in \mathbb{R}^{n_{\text{state}}}.
\)
Hence the linear surrogate acts directly on the original state \cite{proctor2016dynamic}:
\[
\mathbf{s}_{t} \approx \mathbf{A}\,\mathbf{s}_{t-1} + \mathbf{B}\,\mathbf{x}_{t}.
\]
The step-by-step numerical procedure of $\mathbf{A}$ and $\mathbf{B}$ when the identity observable (unity lift) is presented in Appendix I.
\subsection{$\ell_2$ Norm, $Z$-transform, $\text{H}_\infty$ Norm}
The $\ell_2$ norm of a discrete-time signal $\mathbf{x}_k$ ($k\ge 0$) is given by \(
\|\mathbf{x}\|_{\ell_2}
\triangleq
\left(\sum_{k=0}^{\infty}\|\mathbf{x}_k\|_2^2\right)^{1/2},\)
where $\|\mathbf{x}_k\|_2$ denotes the Euclidean norm at time $k$.

The $Z$-transform is a standard tool for analyzing discrete-time systems by mapping the time-domain difference equations into algebraic relations in the $z$-domain. For a causal discrete-time signal $\mathbf{x}_k$ ($k\ge 0$), its $Z$-transform is  
\(Z\{\mathbf{x}_k\}=\mathbf{x}_z=\sum_{k=0}^{\infty}\mathbf{x}_k z^{-k}.\)

The $\text{H}_\infty$ norm is a classical robust-control metric that quantifies the worst-case gain of a transfer function over frequency. For a transfer function $\mathbf{T}_z$, it is defined as  
\(
\|\mathbf{T}_z\|_{\text{H}_\infty}=\sup_{\omega\in[0,\pi]}\sigma_{\max}\!\left(\mathbf{T}_z\!\left(e^{j\omega}\right)\right),
\)
where $\sigma_{\max}(\cdot)$ denotes the maximum singular value. For a stable transfer function $\mathbf{T}_z$ mapping $\mathbf{w}_z$ to $\mathbf{y}_z$, the $\mathcal{H}_\infty$ norm equals the induced $\ell_2$ gain, i.e. \({\|\mathbf{y}\|_{\ell_2}} \le \|\mathbf{T}_z\|_{\mathcal{H}_\infty}{\|\mathbf{w}\|_{\ell_2}}\). Stability is equivalent to all poles of $\mathbf{T}_z$ lying strictly inside the unit circle.
\subsection{Linear Inequality Matrix (LMI) Optimization} 
An LMI is a convex constraint of the form
\(\mathbf{F}(\mathbf{x})=\mathbf{F}_0+\sum_{i} x_i \mathbf{F}_i \succ 0,\)
where $\mathbf{F}_i$ are symmetric matrices and $\mathbf{x}$ is the decision vector. Many stability and robust-performance requirements in control can be expressed as LMIs, enabling an efficient optimization. For a discrete-time linear time-invariant system with disturbance input $\mathbf{w}_t$ and performance output $\boldsymbol{\kappa}_t$, the Bounded Real Lemma provides an equivalent convex condition (expressed as an LMI) ensuring (i) internal stability (i.e. the system's states remain bounded) and (ii) an $\text{H}_\infty$ performance bound $\|\mathbf{T}_{\mathbf{w}\to\boldsymbol{\kappa}}\|_{\text{H}_\infty}<\gamma$, i.e. the induced gain from $\mathbf{w}_t$ to $\boldsymbol{\kappa}_t$ is at most $\gamma$ over all frequencies. This converts the worst-case robustness specifications into tractable LMI constraints \cite{scherer2000linear}.

\section{Interpretable Model for LSTM Networks and Modeling Domain Shifts }
In this section, we model the dynamical behavior of an LSTM network as a closed-loop feedback system, based on the rules governing the LSTM cell's internal dynamics. Specifically, the evolution of  hidden and cell states in a trained LSTM network is represented as a discrete-time nonlinear dynamical system. We then model the causal classes of domain shift (i.e. covariate shift and concept drift) as random additive disturbances to the inputs\footnote{We consider concept shift, in which, given the input features, both the input distribution and the conditional output distribution change.}.
 %To address the challenge of the LSTM network's lack of interpretability and transparency, we apply Koopman operator theory and Dynamic Mode Decomposition with control (DMDc). 
Next, we use Koopman operator and DMDc to construct an interpretable model of the LSTM’s behavior.
%To address the challenge of the LSTM network's lack of interpretability and transparency, we apply Koopman operator theory and Dynamic Mode Decomposition with control (DMDc) to construct an interpretable and transparent model of the LSTM’s behavior.

\subsection{Modeling Dynamical Behavior of an LSTM Network}
%The states of a dynamical system evolve over time based on a set of rules — often expressed as differential or difference equations — in response to external inputs and the system’s previous state.
A dynamical system is a system in which the states evolve over time according to a set of rules. This evolution depends on the system’s previous state. For non-autonomous systems, it also depends on the external inputs. The set of rules is often expressed as differential or difference equations. Like a dynamical system, the internal states of an LSTM network evolve over time in response to a sequence of inputs and their previous internal states.
As illustrated in Fig.~\ref{LSTM_model}a, at each time step $t \in \{0,1,...,T\}$, one LSTM cell takes input vector $\mathbf{x}_t \in \mathbb{R}^{n_{\text{in}}}$, previous hidden state $\mathbf{h}_{t-1} \in \mathbb{R}^{n_{\text{hid}}}$, previous cell state $\mathbf{c}_{t-1} \in \mathbb{R}^{n_{\text{hid}}}$, and computes new hidden state $\mathbf{h}_t \in \mathbb{R}^{n_{\text{hid}}}$, and new cell state $\mathbf{c}_t \in \mathbb{R}^{n_{\text{hid}}}$. Specifically, the evolution of time-dependent  network states ($\mathbf{c}_{t}$ and $\mathbf{h}_{t}$) are influenced by both the current input ($\mathbf{x}_{t}$) and the previous states ($\mathbf{c}_{t-1}$ and $\mathbf{h}_{t-1}$), making the network inherently dynamic. The set of rules that govern the dynamics of the LSTM cell is as follows:
%\begin{figure}[t!] % [h!]%
 %   \centering
 %   \includegraphics[width=3.5 in]{Fig1.png}
 %   \caption{Architecture of an LSTM cell }
  %  \label{LSTM_arc}
%\end{figure}
\begin{figure}[!t]
  \centering
  \begin{minipage}[t]{0.48\textwidth}
    \centering
    \includegraphics[width=\linewidth]{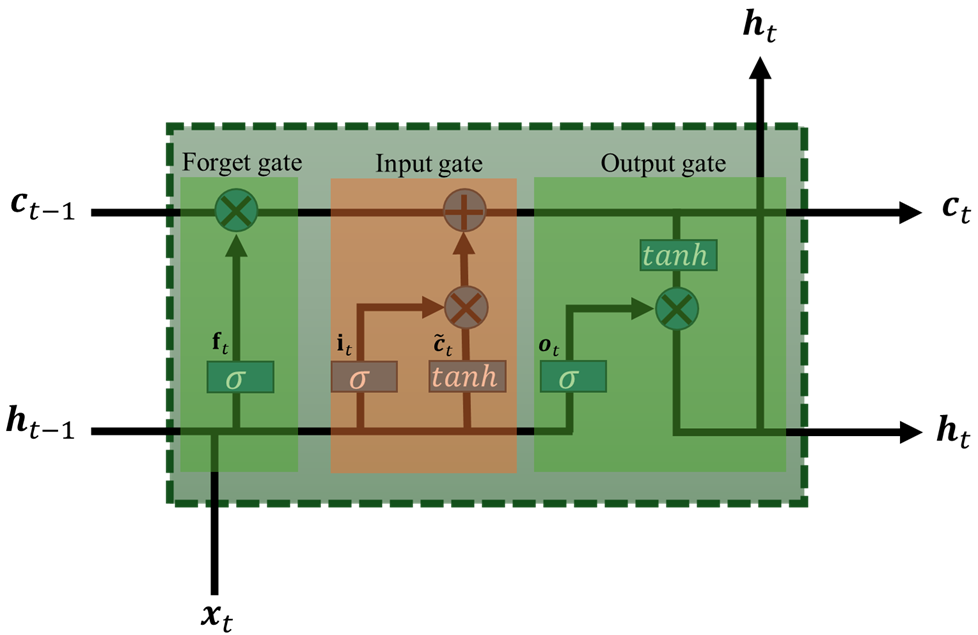}
    \footnotesize{(a) Architecture of an LSTM cell} 
    \label{LSTM_arc}
  \end{minipage}
  \hfill
  \begin{minipage}[t]{0.48\textwidth}
    \centering
    \includegraphics[width=\linewidth]{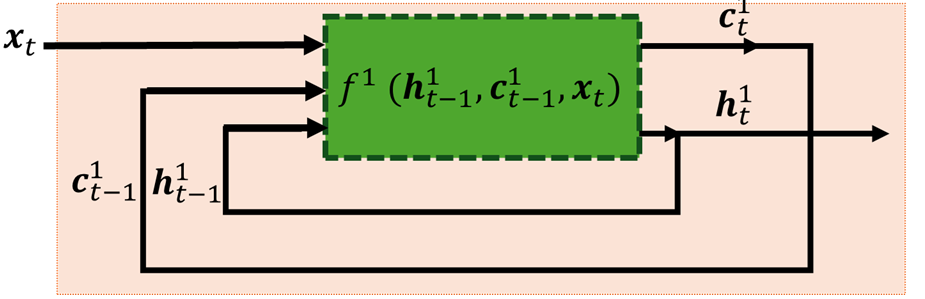}
   \footnotesize{(b) Dynamical behavior model of one layer LSTM}
    \label{LSTM_model_}
  \end{minipage}
  \caption{LSTM model}%: (a) LSTM cell, (b) LSTM network.}
  \label{LSTM_model}
\end{figure}
\begin{itemize}
    \item {\em Forget gate}: 
\begin{equation}
\mathbf{f}_t = \sigma(\mathbf{W}_f \mathbf{x}_{t} + \mathbf{U}_f \mathbf{h}_{t-1} + \mathbf{b}_f),
\label{forget}
\end{equation}
\noindent where $\mathbf{f}_t \in \mathbb{R}^{n_{\text{hid}}}$, $\mathbf{W}_f \in \mathbb{R}^{n_{\text{hid}} \times n_{\text{in}}}$, $\mathbf{U}_f \in \mathbb{R}^{n_{\text{hid}} \times n_{\text{hid}}}$, $\mathbf{b}_f \in \mathbb{R}^{n_{\text{hid}}}$, and $\sigma$ denotes the sigmoid function.
%The forget gate decides how much of the previous cell state $c_{t-1}$ to retain.
\item {\em Input gate}:
\begin{equation}
\mathbf{i}_t = \sigma(\mathbf{W}_i \mathbf{x}_{t} + \mathbf{U}_i \mathbf{h}_{t-1} + \mathbf{b}_i),
\label{input}
\end{equation}
\noindent where $\mathbf{i}_t \in \mathbb{R}^{n_{\text{hid}}}$, $\mathbf{W}_i \in \mathbb{R}^{n_{\text{hid}} \times n_{\text{in}}}$, $\mathbf{U}_i \in \mathbb{R}^{n_{\text{hid}} \times n_{\text{hid}}}$, and $\mathbf{b}_i \in \mathbb{R}^{n_{\text{hid}}}$.
%The input gate controls how much new information is added from the candidate cell state.
\item {\em Candidate cell state}:
\begin{equation}
\tilde{\mathbf{c}}_t = \tanh(\mathbf{W}_c \mathbf{x}_{t} + \mathbf{U}_c \mathbf{h}_{t-1} + \mathbf{b}_c),
\label{candidate}
\end{equation}
\noindent where $\tilde{\mathbf{c}}_t \in \mathbb{R}^{n_{\text{hid}}}$, $\mathbf{W}_c \in \mathbb{R}^{n_{\text{hid}} \times n_{\text{in}}}$, $\mathbf{U}_c \in \mathbb{R}^{n_{\text{hid}} \times n_{\text{hid}}}$, $\mathbf{b}_c \in \mathbb{R}^{n_{\text{hid}}}$, and $\tanh$ denotes the $\tanh$ activation function. 
%This is the new candidate content to be added to the memory.
\item {\em Cell state update}:
\begin{equation}
\mathbf{c}_t = \mathbf{f}_t \odot \mathbf{c}_{t-1} + \mathbf{i}_t \odot \tilde{c}_t,
\label{update}
\end{equation}
\noindent where $\odot$ denotes the Hadamard product.
%The new cell state is a combination of retained old memory and new candidate information.
\item {\em Output gate}:
\begin{equation}
\mathbf{o}_t = \sigma(\mathbf{W}_o \mathbf{x}_{t} + \mathbf{U}_o \mathbf{h}_{t-1} + \mathbf{b}_o),
\label{output}
\end{equation}
\noindent where $\mathbf{o}_t \in \mathbb{R}^{n_{\text{hid}}}$, $\mathbf{W}_o \in \mathbb{R}^{n_{\text{hid}} \times n_{\text{in}}}$, $\mathbf{U}_o \in \mathbb{R}^{n_{\text{hid}} \times n_{\text{hid}}}$, and $\mathbf{b}_o \in \mathbb{R}^{n_{\text{hid}}}$.
%The output gate determines what part of the cell state should be exposed to the hidden state.
\item {\em Hidden state}:
\begin{equation}
\mathbf{h}_t = \mathbf{o}_t \odot \tanh(\mathbf{c}_t).
\label{hidden}
\end{equation}
\end{itemize}
%The states of a dynamical system evolves over time based on a set of rules , often differential or difference equations, in response to external forces and its internal state. LSTM networks can be modeled as dynamical functions because they evolve their internal states over time in response to a sequence of inputs, just like a dynamic system evolves in response to external forces and its internal state. %LSTM networks behavior can be modeled as a dynamical function because it involves the processing of sequential data over time. LSTM networks involve the processing of sequential data over time. As illustrated in Fig. \ref{LSTM_arc}, the processing of sequential data in a LSTM network is done with memory cells that store and update information based on input, forget previous states, and produce outputs. These operations create a time-dependent states (cell and hidden states) evolution, where the next states are influenced by both the current input and the previous states, making the system inherently dynamic. More specifically, the recurrent structure of the LSTM networks, along with their gating mechanisms—input, forget, and output gates, allow them to capture temporal dependencies, similar to a dynamical function evolving according to a set of rules or equations.
\begin{figure*}[!ht] % [h!]%
    \centering
    \includegraphics[width=\textwidth]
    {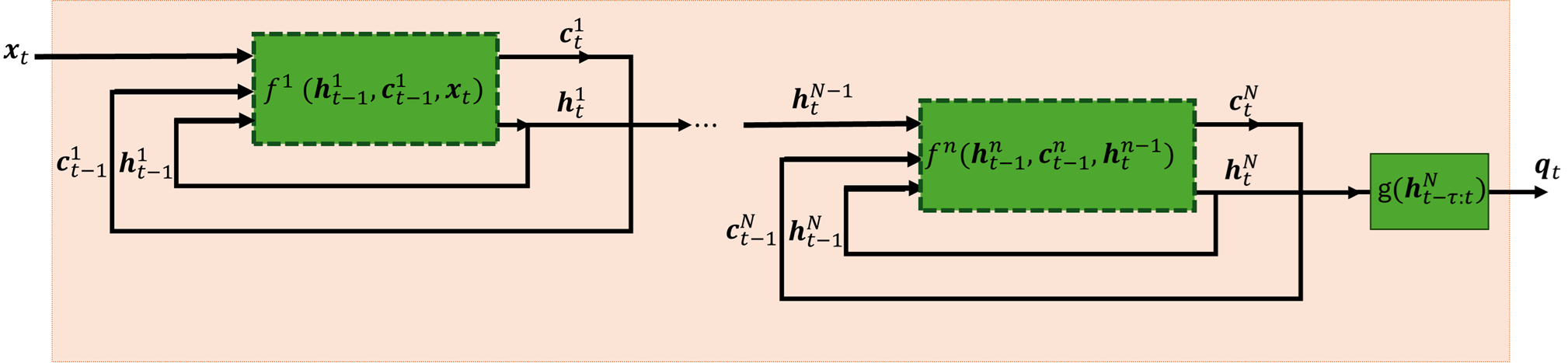}
    \caption{Deep LSTM model}
    \label{DeepLSTM_model}
\end{figure*}

These operations result in the evolution of the time-dependent network states (cell and hidden states),
where the next states are influenced by both the current
input and the previous states. %, making the network inherently dynamic.
Accordingly, we model the dynamical behavior of a one-layer LSTM network as a closed-loop feedback system (shown in Fig.~\ref {LSTM_model}b): %, to capture the dynamic behavior of evolving internal states in response to the current input and its previous internal states.
\begin{align}
 \begin{bmatrix}
    \mathbf{c}^1_t \\
    \mathbf{h}^1_t 
\end{bmatrix} &= f^1\left(\begin{bmatrix}
    \mathbf{c}^1_{t-1} \\
    \mathbf{h}^1_{t-1} 
\end{bmatrix}, \mathbf{x}_{t}\right),  \label{sys_s_1}
\end{align}
where $f^1$ is a nonlinear function (stochastic or deterministic). A deep LSTM model (Fig.~\ref{DeepLSTM_model}) is built by stacking several LSTM layers, where the outputs of each layer become the inputs to the layer above \cite{graves2013speech}. Consequently, the  rules of the evolution of hidden and cell states in a trained deep LSTM network with $N$ layers can be modeled as a discrete dynamical nonlinear system:
\begin{align}
 \begin{bmatrix}
    \mathbf{c}^1_t \\
    \mathbf{h}^1_t 
\end{bmatrix} &= f^1\left(\begin{bmatrix}
    \mathbf{c}^1_{t-1} \\
    \mathbf{h}^1_{t-1} 
\end{bmatrix}, \mathbf{x}_{t}\right), \nonumber \\
\begin{bmatrix}
    \mathbf{c}^2_t \\
    \mathbf{h}^2_t 
\end{bmatrix} &= f^2\left(\begin{bmatrix}
    \mathbf{c}^2_{t-1} \\
    \mathbf{h}^2_{t-1} 
\end{bmatrix}, \mathbf{h}^1_{t}\right), \nonumber \\
\;\;\;\;\;\;\;\;\;\;\;\; \vdots \nonumber \\
\begin{bmatrix}
    \mathbf{c}^N_t \\
    \mathbf{h}^N_t 
\end{bmatrix} &= f^N\left(\begin{bmatrix}
    \mathbf{c}^N_{t-1} \\
    \mathbf{h}^N_{t-1} 
\end{bmatrix}, \mathbf{h}^{N-1}_{t}\right), \nonumber \\
\mathbf{q}_t &= g(\mathbf{h}^{N}_{t-\tau:t}),
    \label{sys_s_2}
\end{align}
where $n \in \{1,\dots,N\}$ denotes the number of LSTM layers, $f^n$s are nonlinear functions (stochastic or deterministic), $g$ is a known or unknown function that maps the hidden-state window
$\mathbf{h}^N_{t-\tau:t} := \begin{bmatrix}{\mathbf{h}^N_{t-\tau}}^\top, & \cdots &, {\mathbf{h}^N_t}^\top\end{bmatrix}^\top
$ to the output $\mathbf{q}_t \in \mathbb{R}^{n_{\text{out}}}$ (stochastic or deterministic), and $\tau \in \mathbb{N}_0$ is the history length ((i.e. window size). When $\tau = 0$, the output map reduces to $\mathbf{q}_t = g(\mathbf{h}^N_t)$. %where $f$ is a nonlinear function (stochastic or deterministic), $g$ is a known or unknown function that maps hidden state $\mathbf{h}_{t-\tau:t}$ to the output $\mathbf{q}_t\in \mathbb{R}^{n_{\text{out}}}$, either stochastic or deterministic, and $\tau \in \mathbb{N}$ is history length (window size). $\tau =0$ reduce output map to $\mathbf{q}_t = g(\mathbf{h}_{t}$.
In addition,  we consider that the dataset \( S_i \)  used to train the LSTM network consists of sequences of the form
$(\mathbf{x}_{0:T}, \mathbf{y}_{0:T'})$. %where \( \mathbf{y}_{t} \) is the corresponding response to \( \mathbf{x}_{t} \) at time step \( t \).
Each sequence \( (\mathbf{x}_{0:T}, \mathbf{y}_{0:T'}) \sim D_i \) is assumed to be drawn from an underlying data-generating distribution \( \mathcal{P}_i \), which defines the joint distribution over the entire sequence. The LSTM network is trained so that the distribution of the network outputs approximates the conditional distribution \( \mathcal{P}^{y|x}_i(\mathbf{y}_{0:T^{'}} \mid \mathbf{x}_{0:T})\). % = \prod _{t=o}^{T}\bar{\mathcal{P}}^{y|x}_i(\mathbf{y}_{t} \mid \mathbf{x}_{0:T}, \mathbf{y}_{0:t}) \), i.e. how the current output depends on the past inputs and outputs. %is assumed to be drawn from an underlying data-generating distribution \( \mathcal{P}_i \). 
%We consider $D_i \subset \mathcal{D}$ is a dataset includes sequences of the form $\mathcal{D}_i = ((\mathbf{x}_{i,0},\mathbf{y}_{i,0}),(\mathbf{x}_{i,1},\mathbf{y}_{i,1}),...,(\mathbf{x}_{i,T},\mathbf{y}_{i,T}))$, where $\mathbf{y}_{i,t}$ is the relevant response of $\mathbf{x}_{i,t}$ in the dataset and each sequence $(\mathbf{x}_0,\mathbf{y}_0) \in D_i$ is drawn from an underlying data-generating distribution $\mathcal{P}_i$. 
%This distribution, $\mathcal{P}_i$, governs the temporal structure of the sequence, and the LSTM network is trained to model the conditional dependencies across time steps.
Additionally, the distribution over input $(\mathbf{x}_{0}, ...,  \mathbf{x}_{T})$ follows $\mathcal{P}^x_i$.
%\begin{figure}[t!] % [h!]%
%    \centering
%    \includegraphics[width=3.5 in]{Fig2.png}
%    \caption{Dynamical behavior model of LSTM}
%    \label{LSTM_model}
%\end{figure}

\vspace{0.2cm}
\noindent
\textbf{Assumption 1}. We assume that $\mathcal{P}_i$, $\mathcal{P}^{y|x}_i$, and $\mathcal{P}^x_i$ are unknown distributions. %, where the expected values of the distributions are denoted by $\overline{m}_i$ and $\overline{m}^x_i$, respectively.

\vspace{0.2cm}
\noindent
\textbf{Assumption 2}. If the LSTM network is a Bayesian LSTM, then the function \( f \) is stochastic; otherwise, it is deterministic. In this work, we focus on the conventional (non-Bayesian) form of the LSTM network.
%If the LSTM network is a Bayesian LSTM network, then the function $f$ is stochastic; otherwise, it is a deterministic function. In this work, we focus on conventional type of LSTM network.
%\textbf{Note 2}. The LSTM model output can either be a sequence of predictions, $\mathbf{y}_t$, (e.g., for time-series forecasting), or a single output at the final time step, $\mathbf{y}_T$, (e.g., for predicting the future channel state based on a sequence of past channel observations). In this work, we consider the first case as the general one.

\vspace{0.2cm}
\noindent
Note that the function $f^n$ is treated as a known nonlinear mapping if the LSTM network's weights, biases, and activation functions (as defined in equations (\ref{forget})–(\ref{hidden})) are available. This also requires that their scale is manageable enough for the behavior of an individual cell to be estimated or approximated. Otherwise, $f^n$ is considered unknown.
%The function $f$ is considered as a known nonlinear mapping if the LSTM’s weights and activation functions (as defined in equations (\ref{forget})–(\ref{hidden})) are known; otherwise, it is considered unknown.
%\vspace{0.2cm}
%We rewrite equation (\ref{sys_s}) as:
%\begin{equation}
% \mathbf{o}_t= f( \mathbf{o}_{t-1}, \mathbf{x}_{t}), 
%    \label{sys_s_}
%\end{equation}
%where 
%\begin{equation}
%\mathbf{o}_t=
%\begin{bmatrix}
%    \mathbf{c}_t \\
%    \mathbf{h}_t 
%\end{bmatrix}. 
%\end{equation}
\begin{figure}[!ht] % [h!]%
    \centering
    \includegraphics[width=3.5 in]{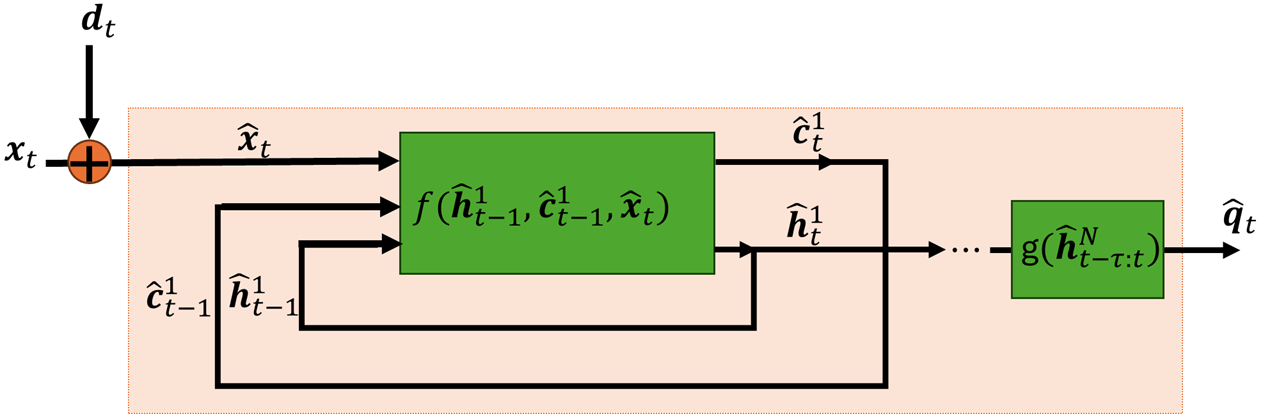}
    \caption{Impact of domain shift on dynamical behavior of an LSTM model}
    \label{LSTM_model_d}
\end{figure}
\subsection{Modeling Domain Shift }

We use the additive disturbance to model the considered domain shift. Specifically, we model domain shifts over the probability distribution $\mathcal{P}^x_i$ by an additive disturbance vector $\mathbf{d}_t$, as shown in Fig. \ref{LSTM_model_d}. The distribution over $(\mathbf{d}_0, ...,  \mathbf{d}_T)$ follows $\mathcal{P}_d$. %, where $\mathcal{P}_d$ is a random disturbance vector.
Accordingly, the discrete dynamical
nonlinear model associated with the evolution of internal states in (\ref{sys_s_2}) is modified as: 
\begin{align}
    \begin{bmatrix}
    \mathbf{\hat{c}}^1_t \\
    \mathbf{\hat{h}}^1_t 
\end{bmatrix} &= f^1\left( \begin{bmatrix}
    \mathbf{\hat{c}}^1_{t-1} \\
    \mathbf{\hat{h}}^1_{t-1} 
\end{bmatrix}, \mathbf{x}_{t}+\mathbf{d}_{t}\right),\nonumber\\ 
\vdots \nonumber \\
    \begin{bmatrix}
    \mathbf{\hat{c}}^N_t \\
    \mathbf{\hat{h}}^N_t 
\end{bmatrix} &= f^N\left( \begin{bmatrix}
    \mathbf{\hat{c}}^N_{t-1} \\
    \mathbf{\hat{h}}^N_{t-1} 
\end{bmatrix}, \mathbf{\hat{h}}^{N-1}_{t}\right),\nonumber\\ 
\mathbf{\hat{q}}_t &= g(\mathbf{\hat{h}}^N_{t-\tau:t})
    \label{sys_s_d_2}
\end{align} 
where $\mathbf{d}_t \in \mathbb{R}^{n_{\text{in}}}$. Specifically, we have $\mathbf{\hat{x}}_{t} = \mathbf{x}_{t}+\mathbf{d}_{t}$, where $(\mathbf{\hat{x}}_0, ...,  \mathbf{\hat{x}}_T)\sim \mathcal{\hat{P}}^x_i$, $\mathcal{\hat{P}}^x_i = \mathcal{P}^x_i \circledast \mathcal{P}_d$, and $\circledast $ denotes convolution operator.%, each component ${d}^i_t$ being drawn from a particular distribution of $D_{w}^i$.  

\vspace{0.2cm}
\noindent
\textbf{Assumption 3}. We assume that the additive disturbance vector $\mathbf{d}_t$ is independent of $\mathbf{x}_t$ at all $t$. In addition, we assume $\mathcal{P}_d$ is an unknown distribution. %, where the expected value of the distribution is denoted by $\overline{m}_d$.

\subsection{Using Koopman Operator and DMDc to Approximate an LSTM Network Model}

In Subsection III.A, we have modeled the evolution of internal states of a deep LSTM network, $\mathbf{c}^n_t$ and $\mathbf{h}^n_t$, as the discrete dynamical nonlinear system (\ref{sys_s_2}). Even with full access to the exact weights, biases, and activation functions of the trained LSTM cells (i.e. known functions $f^n$), %, which means the function $f$ being known,
it remains difficult to interpret the patterns of data recognized by the network. This lack of interpretability primarily stems from the repeated application of nonlinear transformations at each time step $t$ %(as given by equations (\ref{forget})–(\ref{hidden})) 
and potentially from the high dimensionality of the internal states. Consequently, it is challenging to explain how specific inputs lead to particular outputs or to trace the internal computations of the network. To address this issue, we apply data-driven identification methods to approximate the nonlinear function $f^n$ with a linear representation. Specifically, we first use Koopman operators that act on the space of observable functions of the network's states in (\ref{sys_s_2}), allowing the nonlinear dynamics to be analyzed through a linear perspective. Then, we employ DMDc to approximate the Koopman operators.
%As a result, it remains difficult to explain how and why a specific input leads to a particular output. %Even with full access to the exact weights and activation functions of the LSTM and function $f$ being known, we still cannot precisely interpret the patterns the network recognizes or explain why a specific input results in a particular output. 
%As a result, tracing the network’s internal computations or understanding the rationale behind its predictions becomes extremely challenging. Therefore, %in Subsection III.A, we first presented the dynamical evolution of the internal states as a feedback system. Then, in this subsection, we apply data-driven identification methods to approximate the function $f$ by a linear function. Specifically, we first use Koopman operators that act on the space of observable functions of the system's states in (\ref{sys_s_1}), allowing the nonlinear dynamics to be analyzed through a linear perspective. Then, we employ DMD to approximate the Koopman operators.
% Extended state-space Koopman setup for LSTM layer n
% Dimensions and spaces
% H_n: hidden size of layer n (for both c_t^n and h_t^n)
% m_n: input dimension to layer n (m_1 = D for embeddings; m_n = H_{n-1} for n>=2)
%\newcommand{\Hn}{H_n}
%\newcommand{\mn}{m_n}

As discussed in Section~II.C, the Koopman operator, although originally formulated for autonomous systems, has been extended for application to non-autonomous systems \cite{williams2016extending,proctor2018generalizing,korda2018linear}. In this work, we adopt the approach proposed in \cite{korda2018linear}. Accordingly, we define the extended state as:
\begin{equation}
\zeta^n_{t} = 
\begin{bmatrix}
\mathbf{c}^n_{t-1} \\
\mathbf{h}^n_{t-1}\\
\boldsymbol{u}^n_t
\end{bmatrix}, n \in \{1,\dots,N\},
\label{extend_state_lstm}
\end{equation}
where $\mathbf{c}_t^n, \mathbf{h}_t^n \in \mathbb{R}^{n_{\text{hid}}^n}$, 
$\boldsymbol{u}_{t}^n := (\mathbf{u}_t^n,\mathbf{u}_{t+1}^n,\ldots) \in \mathcal{U}_n^{\infty}$, 
and the per-step input to layer $n$ is $\mathbf{u}_t^n \in \mathcal{U}_n \subseteq \mathbb{R}^{n_{\text{in}}^n}$. 
We define
\[
\boldsymbol{u}_t^n =
\begin{cases}
\boldsymbol{x}_t, & n = 1,\\
\boldsymbol{h}_t^{\,n-1}, & n \ge 2~,
\end{cases}
\]
where $\boldsymbol{x}_t = (\mathbf{x}_t, \mathbf{x}_{t+1}, \dots) $ and $\boldsymbol{h}^{n-1}_t = (\mathbf{h}^{n-1}_t, \mathbf{h}^{n-1}_{t+1}, \dots)$\footnote{The extended state $\zeta^n_{t}$ is an ordered pair (a tuple) that the first component is $\begin{bmatrix}
\mathbf{c}^n_{t-1} \\
\mathbf{h}^n_{t-1}
\end{bmatrix}$ (finite-dim, $\mathbb{R}^{2n_{\text{hid}}^n}$), the second component is $\boldsymbol{u}^n_t$ (an entire sequence $(\mathbf{u}_t^n,\mathbf{u}_{t+1}^n,\ldots)$ living in the infinite-dimensional space:
\begin{equation*}
  \zeta_t^n \;\in\; \underbrace{\mathbb{R}^{2n_{\text{hid}}^n}}_{\text{finite}} \times
          \underbrace{\mathcal U_n^{\infty}}_{\text{sequence space}}.
  %\label{eq:extended-state}
\end{equation*}}.
Thus, the dynamics of the extended state is described as:
\begin{equation}
 \zeta^n_{t+1} = {F}^n(\zeta^n_{t}) := 
\begin{bmatrix}
f^n\left(\begin{bmatrix}\mathbf{c}^n_{t-1} \\
\mathbf{h}^n_{t-1}\end{bmatrix},\boldsymbol{u}^n_t(0)\right) \\
\ell(\boldsymbol{u}^n_t)
\end{bmatrix},    
\label{extend_dynamics_lstm}
\end{equation}
where $\boldsymbol{u}^n_t(0)$ denotes the first element of input sequence in $\boldsymbol{u}^n_t$ and $\ell(\boldsymbol{u}^n_t)= \boldsymbol{u}^n_{t+1}$. The Koopman operator \( \mathcal{K} : \mathcal{O} \to \mathcal{O} \), associated with the system in (\ref{extend_dynamics_lstm}) is defined as:
\begin{equation}
\mathcal{K}^n \varphi^n(\zeta^n_{t}) = \varphi^n(F^n(\zeta^n_t)),
\label{extend_koopman_lstm}
\end{equation}
where \( \varphi^n \) is an element of the observable space \( \mathcal{O} \), which is an infinite-dimensional Hilbert space.

To approximate each Koopman operator \( \mathcal{K}^n \) in (\ref{extend_koopman_lstm}), we adopt the DMDc method from \cite{proctor2016dynamic}, yielding a linear approximation of the nonlinear dynamics in (\ref{sys_s_1}):
\begin{equation}
    \boldsymbol{\psi}^n\left(\begin{bmatrix}
\mathbf{c}^n_{t} \\
\mathbf{h}^n_{t}
\end{bmatrix}\right) = \mathbf{A}^n\boldsymbol{\psi}^n\left(\begin{bmatrix}
\mathbf{c}^n_{t-1} \\
\mathbf{h}^n_{t-1}
\end{bmatrix}\right) + \mathbf{B}^n \mathbf{u}^n_{t}(0), 
\end{equation}
where observer $\boldsymbol{\psi}^n$ is defined as $\boldsymbol{\psi}^n\left(\begin{bmatrix}
\mathbf{c}^n_{t} \\
\mathbf{h}^n_{t}
\end{bmatrix}\right) = \begin{bmatrix}
\mathbf{c}^n_{t} \\
\mathbf{h}^n_{t}
\end{bmatrix}
$, which satisfies the requirement that the space of observables is sufficiently rich. Consequently, we can drive the final linear approximation of the nonlinear dynamics in eq. (\ref{sys_s_2}) as:
\begin{equation}
   \begin{bmatrix}
    \mathbf{c}^n_{t} \\
    \mathbf{h}^n_{t} 
\end{bmatrix} = \mathbf{A}^n \begin{bmatrix}
    \mathbf{c}^n_{t-1} \\
    \mathbf{h}^n_{t-1} 
\end{bmatrix} + \mathbf{B}^n \mathbf{u}^n_{t}(0),
    \label{DMDc_d}
\end{equation}
where \( \mathbf{A}^n \in \mathbb{R}^{2n^n_{\text{hid}} \times 2n^n_{\text{hid}}} \) is the Koopman matrix and \( \mathbf{B}^n \in \mathbb{R}^{2n^n_{\text{hid}} \times n^n_{\text{in}}} \) is the input matrix of layer $n$. The numerical procedure for computing \(\mathbf{A}^n\) and \( \mathbf{B}^n \) is discussed in \textbf{Appendix~I}. Equation (\ref{DMDc_d}) presents an interpretable model of the dynamics associated with the evolution of states of each layer in a trained deep LSTM network. Based on the linear representation of layer $n$ (eq.~\eqref{DMDc_d}) and the inter-layer connections, the dynamics governing the deep LSTM network can be expressed as: %the evolution of the hidden and cell states in a trained
\begin{align}
   \begin{bmatrix}
    {\mathbf{s}}^1_{t} \\
    \vdots\\
    {\mathbf{s}}^N_{t}
\end{bmatrix} &= \mathbf{A} \begin{bmatrix}
    {\mathbf{s}}^1_{t-1} \\
    \vdots\\
    {\mathbf{s}}^N_{t-1} 
\end{bmatrix} + \mathbf{B} \mathbf{x}_t,\label{DMDc_deep_total}\\
\mathbf{q}_t &= g(\mathbf{C}\,\mathbf{s}_{t-\tau:t}),
    \label{DMDc_deep_total_2}
\end{align}
where $\mathbf{s}^n_{t}= \begin{bmatrix}
    \mathbf{c}^n_{t} \\
    \mathbf{h}^n_{t} 
\end{bmatrix}$ and matrices $\mathbf{A}$ and $\mathbf{B}$ are defined as:
\begin{align}
\mathbf{A}&=\begin{bmatrix}
\mathbf{A}^1 & 0 & 0 & \dots & 0 \\
\mathbf{B}^2\mathbf{C}^1 & \mathbf{A}^2 & 0 & \dots & 0 \\
0 & \mathbf{B}^3\mathbf{C}^2 & \mathbf{A}^3 & \ddots & \vdots \\
\vdots & \vdots & \ddots & \ddots & 0 \\
0 & 0 & \dots & \mathbf{B}^{N}\mathbf{C}^{N-1} & \mathbf{A}^N
\end{bmatrix},\nonumber \\
\mathbf{B} &= [\mathbf{B}^1,0, \dots, 0 ]^T,
\label{full_matrix}
\end{align}
\noindent where $\mathbf{C}^n$ is a projection matrix selecting the hidden-state block, i.e.
$\mathbf{h}_t^{\,n} = \mathbf{C}^n\,\mathbf{s}_t^{\,n}$ with
$
\mathbf{C}^n = \begin{bmatrix}\mathbf{0}_{n_{\text{hid}}\times n_{\text{hid}}} \;&\; \mathbf{I}_{n_{\text{hid}}}\end{bmatrix}
$ and $\mathbf{0}_{n_{\text{hid}}\times n_{\text{hid}}}$ is the zero matrix, and
$\mathbf{I}_{n_{\text{hid}}}$ is the identity matrix. In addition, since LSTM output $\mathbf{q}_t$ depends only on the hidden states of layer $N$ ($\mathbf{h}^N_{t-\tau:t}$), we define projection matrix $\mathbf{C}$ such that: 
\({\mathbf{h}}^N_t = \mathbf{C} \mathbf{{{s}}}_{t},\) 
where \(\mathbf{{s}}_t=\begin{bmatrix}
    {\mathbf{s}}^1_{t} \\
    \vdots\\
    {\mathbf{s}}^N_{t}
\end{bmatrix} \) and $\mathbf{C} = [\mathbf{0} \quad \mathbf{I}_{n_{\text{hid}}^N}]$, in which $\mathbf{0}$ is a zero matrix of appropriate dimension and $\mathbf{I} \in \mathbb{R}^{n_{\text{hid}}^n\times n_{\text{hid}}^n}$ is an identity matrix.

In Section III.B, we modeled a domain shift as an additive disturbance. To incorporate the domain shifts into the proposed interpretable model, eq. (\ref{DMDc_d}) is adjusted as:
\begin{align}
   \begin{bmatrix}
    \hat{{\mathbf{s}}}^1_{t} \\
    \vdots\\
    \hat{{\mathbf{s}}}^N_{t}
\end{bmatrix} &= \mathbf{A} \begin{bmatrix}
    \hat{\mathbf{s}}^1_{t-1} \\
    \vdots\\
    \hat{\mathbf{s}}^N_{t-1} 
\end{bmatrix} + \mathbf{B} (\mathbf{x}_{t}+\mathbf{d}_{t}),\label{DMDc_d_dis}\\
\hat{\mathbf{q}}_t &= g(\mathbf{C}\,\hat{\mathbf{s}}_{t-\tau:t}),
    \label{DMDc_d_dis_2}
\end{align}
where $\hat{.}$ denotes the states under the domain shift and $\mathbf{d}_t$ presents the domain shift modeled with vector $\mathbf{d}_t\sim \mathcal{P}_d$.

\section{Generalization Analysis of LSTM Networks}

Section~III introduced an interpretable model that characterizes the evolution of the states and their relationship to the output in a trained deep LSTM network. Based on this, we present a method for analyzing the generalization by deriving a GE bound for the trained LSTM network. This bound is obtained by analyzing the LSTM output $\hat{\mathbf{q}}_t$ (eq.~ (\ref{DMDc_d_dis_2})) when the input data is sampled from a distribution that differs from the training distribution.

\subsection{Analyzing the Effect of Domain Shifts on the Hidden State}

In this subsection, we estimate how a domain shift can impact the trained LSTM's states in the worst-case scenario. Specifically, the H$_\infty$ norm is used to evaluate the maximum impact of the domain shift on the trained LSTM's states. %Hereafter, we only focus on the analysis for the probabilistic LSTM as the general case, since in the deterministic case, we have \(\begin{bmatrix}
%    \mathbf{\bar{c}}_{t} \\
%    \mathbf{\bar{h}}_{t} 
%\end{bmatrix} = \begin{bmatrix}
%    \mathbf{c}_{t} \\
%    \mathbf{h}_{t} 
%\end{bmatrix}\). 
%Our main aim is to evaluate the generalization bound by analyzing the LSTM output $\mathbf{q}_t$. 

Now, in the first step, we transfer the models (\ref{DMDc_deep_total}) and (\ref{DMDc_d_dis}) into the \( Z \)-domain. These models (\ref{DMDc_deep_total}) and (\ref{DMDc_d_dis}) transformed in \( Z \)-domain are:
\begin{align}
   z\mathbf{{s}}_z - z\mathbf{{s}}_{t=0} &= \mathbf{A} \mathbf{{s}}_z + \mathbf{B} \mathbf{x}_z, \label{DMDc_z}\\
    z\hat{{\mathbf{s}}}_z - z\hat{{\mathbf{s}}}_{t=0} &= \mathbf{A} \hat{{\mathbf{s}}}_z + \mathbf{B} \mathbf{x}_z + \mathbf{B} \mathbf{d}_z. \label{DMDc_d_z}
\end{align}
Accordingly, the transfer function from $\mathbf{d}_z$ to $\hat{\mathbf{{s}}}_z$ and $\mathbf{{s}}_z$ can be calculated as:
\begin{align*}
   (z\mathbf{I}- \mathbf{A}) \mathbf{{s}}_z &= z\mathbf{{s}}_{t=0} + \mathbf{B} \mathbf{x}_z,\\
   \mathbf{{s}}_z &= (z\mathbf{I}- \mathbf{A})^{-1} (z\mathbf{{s}}_{t=0} + \mathbf{B} \mathbf{x}_z), \; \; \mbox{and} \\
%\end{align*}
%\begin{align*}
(z\mathbf{I}-\mathbf{A}) \hat{{\mathbf{s}}}_z &=   z\hat{{\mathbf{s}}}_{t=0} + \mathbf{B} \mathbf{x}_z + \mathbf{B} \mathbf{d}_z,\\
    \hat{{\mathbf{s}}}_z &=  (z\mathbf{I}-\mathbf{A})^{-1}( z\hat{{\mathbf{s}}}_{t=0} + \mathbf{B} \mathbf{x}_z + \mathbf{B} \mathbf{d}_z).
\end{align*}
\vspace{-5pt}

\noindent\textbf{Assumption 4}. We assume that $\hat{{\mathbf{s}}}_{t=0} = {\mathbf{{s}}}_{t=0}$.\\
\vspace{-5pt}

With the above assumption, the transfer function from $\mathbf{d}_z$ to $\hat{{\mathbf{s}}}_z-{\mathbf{s}}_z$ can be calculated as:
\begin{align}
    \hat{{\mathbf{s}}}_z - \mathbf{{s}}_z &=  (z\mathbf{I}-\mathbf{A})^{-1}( \mathbf{B} \mathbf{d}_z),\label{diff}\\ 
    {\mathbf{T}}_z^{ds} &= (z\mathbf{I}-\mathbf{A})^{-1}\mathbf{B}.
    \label{transfer_z}
\end{align} 
Consequently, we can calculate the H\( _\infty \) norm of the transfer function \({\mathbf{T}}_z^{ds} \) as:
\begin{equation}
\| {\mathbf{T}}_z^{ds}  \|_{\text{H}_\infty} = \sup_{\omega \in [0, \pi]} \sigma_{\text{max}}\left( {(e^{j\omega}\mathbf{I} - \mathbf{A})^{-1}\mathbf{B}} \right).
\end{equation}
Considering projection matrix $\mathbf{C}$, the H\( _\infty \) norm of the transfer function from $\mathbf{d}_z$ to $\hat{{\mathbf{h}}}^N_z-{\mathbf{h}}^N_z$ (\({\mathbf{T}}_z^{dh} \)) is obtained as:
\begin{equation}
\| {\mathbf{T}}_z^{dh}  \|_{\text{H}_\infty} = \sup_{\omega \in [0, \pi]} \sigma_{\text{max}}\left( \mathbf{C}{(e^{j\omega}\mathbf{I} - \mathbf{A})^{-1}\mathbf{B}} \right),
\end{equation}
where ${\mathbf{T}}_z^{dh} = \mathbf{C}{(z\mathbf{I} - \mathbf{A})^{-1}\mathbf{B}}$ is the transfer function from $\mathbf{d}_z$ to $\hat{{\mathbf{h}}}^N_z-{\mathbf{h}}^N_z$.

\noindent
\textbf{Assumption 5}. We assume that $\mathbf{T}_z^{dh}$ is a stable transfer function, i.e. all poles of $\mathbf{T}_z^{dh}$ lie strictly inside the unit circle. %Consequently, bounded disturbances $\mathbf{d}_t$ lead to bounded outputs (e.g., the induced deviation $\Delta\mathbf h_t^N := \hat{\mathbf h}_t^{\,N} - \mathbf h_t^{\,N}$). 
This assumption is reasonable because a well-trained LSTM should exhibit a stable state evolution on its operating trajectories.

%\textbf{Assumption 4}. We assume that the domain change $\mathbf{d}_t$ has bounded energy, i.e. for any realization:
%\begin{align*}
%    \sum_{t=0}^{T} (\|\mathbf{d}_t\|_2)^2  \leq \alpha.
%\end{align*}
%We assume that the energy of any realization of domain change $\mathbf{d}_t$ is limited and $\mathbf{d}_t$ satisfies \( \sum_{t=0}^{T} (\|\mathbf{d}_t\|_2)^2  \leq \alpha \) 
\vspace{2pt}
\noindent \begin{theorem}
For a given trained LSTM, suppose that the domain shift $\{\mathbf d_t\}_{t=0}^T$ satisfies $\|\mathbf d\|_{\ell_2}\le \alpha$.
Let the linear map from $\mathbf d_z$ to $\Delta\mathbf h_z^N$
is the  $\mathbf T^{dh}_z$, then
\begin{align}
\sum_{t=0}^{T} \big\|\hat{\mathbf h}^{\,N}_t - \mathbf h^{\,N}_t\big\|_2^2
&\;\le\; \|\mathbf T^{dh}_z\|_{H_\infty}^{\,2}\,\alpha^{2}, \label{eq:l2-sum-bound}\\[4pt]
\max_{0\le t\le T}\big\|\hat{\mathbf h}^{\,N}_t - \mathbf h^{\,N}_t\big\|_2
&\;\le\; \|\mathbf T^{dh}_z\|_{H_\infty}\,\alpha. \label{eq:max-bound}
\end{align}
\end{theorem}

\begin{proof}
See \textbf{Appendix II}.
\end{proof}
\noindent \textbf{Interpretation of $\| {\mathbf{d}} \|_{\ell_2} \leq \alpha$ using Wasserstein distance}:\\
We model the domain shift by an additive disturbance
\(\hat{\mathbf x}_t=\mathbf x_t+\mathbf d_t\), therefore the shifted input law is the
distribution of \(\hat{\mathbf x}_{0:T}=\mathbf x_{0:T}+\mathbf d_{0:T}\): 
\[
(\hat{\mathbf x}_0,\dots,\hat{\mathbf x}_T)\sim \hat{\mathcal P}^x_i,
\]
where 
\(\hat{\mathcal P}^x_i=\mathcal P^x_i \!\circledast\! \mathcal P_d\text{.}
\) Consider \(W_1\) on the sequence space with ground metric
\(c(\mathbf x_{0:T},\hat{\mathbf x}_{0:T})=\|\hat{\mathbf x}_{0:T}-\mathbf x_{0:T}\|_{\ell_2}\).
Using the coupling \((\hat{\mathbf x}_{0:T},\mathbf x_{0:T})=(\mathbf x_{0:T}+\mathbf d_{0:T},\,\mathbf x_{0:T})\),
\[
W_1\big(\hat{\mathcal P}^x_i,\mathcal P^x_i\big)
\;\le\; \mathbb E\big[\|\mathbf d_{0:T}\|_{\ell_2}\big].
\]
%Since
%\[
Since  $\|\mathbf d\|_{\ell_2}\le \alpha$, we get $W_1\big(\hat{\mathcal P}^x_i,\mathcal P^x_i\big) \;\le\; \alpha$.
%\]
%then
%\[
%W_1\big(\hat{\mathcal P}^x_i,\mathcal P^x_i\big) \;\le\; \alpha.
%\]
In other words, bounding $\|{\mathbf{d}} \|_{\ell_2}$ directly controls the Wasserstein distance between the target and the training dataset.

\subsection{OOD Generalization Bound for Domain Shifts}

This subsection derives a OOD generalization bound for a trained deep LSTM network under the domain shift. Specifically, we derive a bound on the GE in the presence of domain shift. We use \textbf{Theorem~1}, which upper-bounds the maximum impact of domain shift on the last-layer hidden state.

% Specifically, characterize how well the LSTM has learned the correct input–output mapping and how effectively it can identify relevant input patterns while filtering out irrelevant or noisy information. 
%first assume that the true value of \( \mathbf{y}_{t} \) remains unchanged under the covariate shift (i.e. \(  \hat{\mathbf{y}}_{t} = \mathbf{y}_{t} \)). We then 
%To do this, we use the upper bound on the maximum impact of domain shift on the hidden state of the last layer of the LSTM model, as presented in \textbf{Theorem 1}.

We note that the function \( g \), which maps the hidden state \( \mathbf{h}^N_{t-\tau:t} \) to the output \( \mathbf{q}_t \), can be either a deterministic or a probabilistic function. Also, we make the following assumption.

\vspace{2pt}
\noindent
\textbf{Assumption 6}.
Let $\tau\in\mathbb{N}$ and write the hidden-state window
${\mathbf{h}}^N_{t-\tau:t}:=[{{\mathbf{h}}^N_{t-\tau}}^\top,\ldots,{{\mathbf{h}}^N_t}^\top]^\top$.
Assume $g:\mathbb{R}^{(\tau+1)n^N_h}\to\mathbb{R}^{n_{\text{out}}}$ is Lipschitz
(or Lipschitz almost surely if $g$ is probabilistic), i.e. there exists a constant $G>0$ such that
%\textbf{Assumption 5}: We assume that \( g \) is Lipschitz continuous (or Lipschitz almost surely (a.s.) if it is a probabilistic function, meaning that for almost every outcome with a Lipschitz constant $G$, i.e. 
\[
\| g(\hat{\mathbf{h}}^N_{t-\tau:t}) - g(\mathbf{h}^N_{t-\tau:t}) \|_2 \leq G \cdot \sqrt{\tau+1} \cdot \max_{i\in [t-\tau:t]}
\| \hat{\mathbf{h}}^N_i - \mathbf{h}^N_i \|_2,
\]
\[
\| g(\hat{\mathbf{h}}^N_{t-\tau:t}) - g(\mathbf{h}^N_{t-\tau:t}) \|_2 \leq G \cdot \sqrt{\tau+1} \cdot \|{\mathbf{T}}_z^{dh}\|_{H_\infty} \cdot \alpha.
\]

It is worth noting that, based on the structural knowledge of $g(.)$, we can derive tighter bounds for $\| g(\hat{\mathbf{h}}^N_{t-\tau:t}) - g(\mathbf{h}^N_{t-\tau:t}) \|_2$ that may avoid the extra $\sqrt{\tau+1}$ factor and are therefore tighter. For deriving a generalization bound under domain shifts, we assume the following.

\vspace{2pt}
\noindent
\textbf{Assumption 7}. Assume the domain shift $\mathbf{d}_t$ such that \( \|{\mathbf{d}}\|_{\ell_2} \leq \alpha \) causes outputs to change by at most $\beta$:
\( \|\hat{\mathbf{y}} - \mathbf{y}\|_{\ell_2} \leq \beta.\)

\vspace{2pt}
\noindent
Note that the relationship between $\beta$ and $\alpha$ is application-dependent and can
be derived from domain knowledge about the data-generating mechanism. For example, in our experiments (presented in Section VI), the LSTM predicts the future samples from a history of past samples. This corresponds to a shift operator with induced $\ell_2$ gain equal to $1$, hence $\beta=\alpha$.

%Here, we first focus on the regression task. 
Let \(\mathcal{L}(y_t, q_t)\) denote the regression loss between the true target \(y_t\) and the prediction \(q_t\).

\vspace{2pt}
\noindent
\textbf{Assumption 8}.  The loss function \( \mathcal{L}(y_t, q_t) \) is Lipschitz continuous with respect to its second argument, with Lipschitz constant \( L \); that is,
\[
|\mathcal{L}(\hat{\mathbf{y}}_t, \hat{\mathbf{q}}_t) - \mathcal{L}(\mathbf{y}_t, \mathbf{q}_t)| \leq L \cdot (\|\hat{\mathbf{y}}_t - \mathbf{y}_t\|_2 +
\|\hat{\mathbf{q}}_t - \mathbf{q}_t\|_2).
\]

\vspace{2pt}
\noindent
\begin{corollary}
Assume the Lipschitz constants $G$ and \( L \) for functions $g$ and \( \mathcal{L} \), respectively. Given an LSTM network trained for a regression problem, for any domain shift $\mathbf{d}_t$ such that \( \|{\mathbf{d}}\|_{\ell_2} \leq \alpha \), the generalization error is bounded by:
\begin{align}
\left| \mathbb{E}\left[\mathcal{L}(\hat{\mathbf{y}}_t, \hat{{\mathbf{q}}}_t)\right] - \mathbb{E}\left[\mathcal{L}(\mathbf{y}_t, {\mathbf{q}}_t)\right] \right|   
&\leq L\cdot (\beta+
G \cdot \sqrt{\tau+1} \nonumber\\ &\cdot\|{\mathbf{T}}_z^{dh}\|_{H_\infty} \cdot \alpha). %\nonumber\\ &: \forall t \in \{0,1,...T\},
\end{align} 

\end{corollary}

\begin{proof}
See \textbf{Appendix III}.
\end{proof}

To derive \textbf{Corollary~1}, we have assumed that  the loss function
\( \mathcal{L}(\mathbf{y}_t, \mathbf{q}_t) \) is Lipschitz continuous. However, this assumption may fail in some cases, e.g., when the loss
is calculated as the mean squared error (MSE). Nevertheless, we can obtain a useful bound on the difference in MSE loss
by exploiting the structure of the loss as follows: 
\begin{align*}
    \big|\mathbb E\, [\mathcal{L}(\hat{\mathbf y}_t,\hat{\mathbf q}_t)]-\mathbb E\,[\mathcal{L}(\mathbf y_t,\mathbf q_t)]\big|&\le 2\,\mathbb E\,\sqrt{\mathcal{L}(\mathbf y_t,\mathbf q_t)}\,(\beta\\&+G \cdot \sqrt{\tau+1} \cdot\|{\mathbf{T}}_z^{dh}\|_{H_\infty} \cdot \alpha)\\+(\beta&+G \cdot \sqrt{\tau+1} \cdot\|{\mathbf{T}}_z^{dh}\|_{H_\infty} \cdot \alpha)^2.
\end{align*}
\begin{proof} 
See \textbf{Appendix IV}.
\end{proof}

\vspace{4pt}
%\noindent
%2) {\em Generalization Bound for Classification Tasks}
%\vspace{6pt}
%\paragraph{\textbf{Analysis of the Proposed Bound for Regression}}
%\vspace{6pt}
%\paragraph{\textbf{Classification}}
%We now consider the classification setting.
\section{A Method for Domain Generalization}\label{OOD_method}
\begin{figure}[!t]
    \centering
    \includegraphics[width=3.5 in]{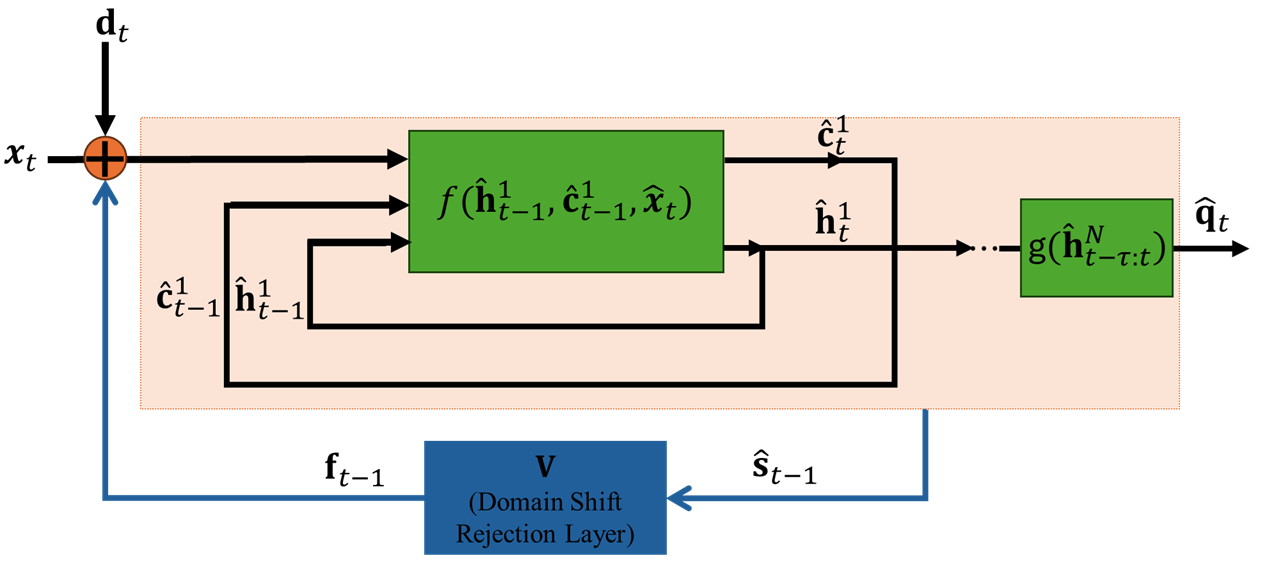}
  \caption{Domain shift rejection architecture of a trained LSTM}%: (a) LSTM cell, (b) LSTM network.}
  \label{Architecture_modified}
\end{figure}

The goal of domain generalization (DG) is to improve the performance on OOD data. In this section, we propose a DG method that reduces the OOD generalization error of a trained LSTM under input distribution shifts. Our approach leverages the interpretable representation in eq.~\eqref{DMDc_d_dis} to construct a dynamic architecture of the trained LSTM.

\vspace{2pt}
\noindent\textbf{Assumption 9}. We consider the domain shift, where the true value of $\mathbf{y}_{t}$ remains unchanged under the shift (i.e. $\hat{\mathbf{y}}_{t} = \mathbf{y}_{t}$).

This assumption holds in several practical settings, e.g., when shifts arise from adversarial perturbations or natural variations that affect the observed inputs while preserving the underlying target.

Specifically, to enhance the robustness against the input disturbances $\mathbf{d}_t$, we augment the trained LSTM with a state-feedback term ${\mathbf{f}}_{t-1}$ (Fig.~\ref{Architecture_modified}). Using eq.~\eqref{DMDc_d_dis}, we design ${\mathbf{f}}_{t-1}$ by solving an LMI-constrained optimization problem that minimizes the effect of $\mathbf{d}_t$ on the LSTM state dynamics, thereby reducing the deviation $\Delta\mathbf{h}_t^N := \hat{\mathbf{h}}_t^{\,N} - \mathbf{h}_t^{\,N}$. Then, considering \textbf{Assumption~7} and \textbf{Assumption~9} (which imply that $\beta = 0$), and based on \textbf{Corollary~1}, we can conclude that the generalization error  decreases for the learned model.

\subsection{Domain Shift Rejection of State Deviation}

Based on the proposed interpretable model in eq.~(\ref{DMDc_deep_total}), the nominal dynamics of the trained LSTM network is:
\begin{align*}
\mathbf{s}_t = &\mathbf{A} \mathbf{s}_{t-1} + \mathbf{B} \mathbf{x}_t,\nonumber\\ 
{\mathbf{h}}^N_t =  &\mathbf{C} \mathbf{{{s}}}_t,
\end{align*}
where \(\mathbf{{s}}_t=\begin{bmatrix}
    {\mathbf{s}}^1_{t} \\
    \vdots\\
    {\mathbf{s}}^N_{t}
\end{bmatrix} \), \(\mathbf{{s}}_t \in \mathbb{R}^{n_{\text{total}}}\), and $\mathbf{C} = [\mathbf{0} \quad \mathbf{I}_{n_{\text{hid}}^N}]$. Under the domain shift with state-feedback control input $\mathbf{f}_{t-1}$ (Fig. \ref{Architecture_modified}), the dynamics of the LSTM  network is given by:
\begin{align}
\hat{\mathbf{s}}_t &= \mathbf{A} \hat{\mathbf{s}}_{t-1} + \mathbf{B}(\mathbf{x}_t + \mathbf{d}_t + \mathbf{f}_{t-1}),\nonumber\\ 
\hat{{\mathbf{h}}}^N_t &= \mathbf{C} \mathbf{\hat{{{s}}}}_t,
\end{align}
where \(\hat{{\mathbf{s}}}_t =\begin{bmatrix}
    \hat{\mathbf{s}}^1_{t} \\
    \vdots\\
    \hat{\mathbf{s}}^N_{t},
\end{bmatrix} \) and the state-feedback input is given by:
\begin{equation}
\mathbf{f}_{t-1} = \mathbf{V} \hat{\mathbf{s}}_{t-1},
\end{equation}
where $\mathbf{V}\in \mathbb{R}^{n_{\text{in}} \times n_{\text{total}}}$ is the modifying matrix to be designed.

To design $\mathbf{V}$, we first define the state deviation between $\hat{\mathbf{s}}_t$ and the nominal states as:
\begin{equation}
\boldsymbol{\delta}_t := \hat{\mathbf{s}}_t - \mathbf{s}_t.
\end{equation}
Substituting the values of $\hat{\mathbf{s}}_t$ and  $\mathbf{s}_t$ yields:
\begin{align}
\boldsymbol{\delta}_t &= \hat{\mathbf{s}}_t - \mathbf{s}_t \nonumber \\
&= \mathbf{A} \hat{\mathbf{s}}_{t-1} + \mathbf{B}(\mathbf{x}_t + \mathbf{d}_t + \mathbf{V} \hat{\mathbf{s}}_{t-1}) - \left( \mathbf{A} \mathbf{s}_{t-1} + \mathbf{B} \mathbf{x}_t \right) \nonumber \\
&= \mathbf{A}(\hat{\mathbf{s}}_{t-1} - \mathbf{s}_{t-1}) + \mathbf{B} \mathbf{d}_t + \mathbf{B} \mathbf{V} \hat{\mathbf{s}}_{t-1} \nonumber \\
&= \mathbf{A} \boldsymbol{\delta}_{t-1} + \mathbf{B} \mathbf{d}_t + \mathbf{B} \mathbf{V} (\mathbf{s}_{t-1} + \boldsymbol{\delta}_{t-1}) \nonumber \\
&= (\mathbf{A} + \mathbf{B} \mathbf{V}) \boldsymbol{\delta}_{t-1} + \mathbf{B} \mathbf{V} \mathbf{s}_{t-1} + \mathbf{B} \mathbf{d}_t.
\label{state_deviation}
\end{align}
%Thus, the dynamics of state deviation can be presented as: 
%\begin{align}
%\boldsymbol{\delta}_t &=(\mathbf{A} + \mathbf{B} \mathbf{V}) \boldsymbol{\delta}_{t-1} + \mathbf{B} \mathbf{V} \mathbf{s}_{t-1} + \mathbf{B} \mathbf{d}_t.
%\label{state_deviation}
%\end{align}
Transferring eq.~(\ref{state_deviation})
into the $Z$-domain:
\[\boldsymbol{\delta}_z = (\mathbf{A} + \mathbf{B}\mathbf{V}) z^{-1} \boldsymbol{\delta}_z 
+ \mathbf{B}\mathbf{V} z^{-1} \mathbf{s}_z + \mathbf{B}\mathbf{d}_z,\]
 \[\left( \mathbf{I} - z^{-1}(\mathbf{A} + \mathbf{B}\mathbf{V}) \right) \boldsymbol{\delta}_z 
= \mathbf{B}\mathbf{V} z^{-1} \mathbf{s}_z + \mathbf{B}\mathbf{d}_z, \]
\[\boldsymbol{\delta}_z = 
\left( \mathbf{I} - z^{-1}(\mathbf{A} + \mathbf{B}\mathbf{V}) \right)^{-1}
\left( \mathbf{B}\mathbf{V} z^{-1} \mathbf{s}_z + \mathbf{B}\mathbf{d}_z \right).\]
Therefore, the transfer function from $\mathbf{d}_z$ to $\boldsymbol{\delta}_z$ is given by:
\[\mathbf{T}^{\delta d}_z = 
\left( \mathbf{I} - z^{-1}(\mathbf{A} + \mathbf{B}\mathbf{V}) \right)^{-1} \mathbf{B}.\]
The goal of using the static state-feedback gain
$\mathbf{V}$ is to attenuate the effect of the disturbance $\mathbf{d}_t$ on $\boldsymbol{\delta}_t$. 
However, the term $\mathbf{B}\mathbf{V}\mathbf{s}_{t-1}$ acts as an additional known signal that also influences $\boldsymbol{\delta}_t$ unintentionally. The transfer function from $\mathbf{s}_z$ to $\boldsymbol{\delta}_z$ is given by:
\[\mathbf{T}^{\delta s}_z
= \left( \mathbf{I} - z^{-1}(\mathbf{A} + \mathbf{B}\mathbf{V}) \right)^{-1}
\, \mathbf{B}\mathbf{V}\, z^{-1}.\]
In the next step, by considering the effect of both $\mathbf{d}_t$ and $\mathbf{s}_z$, we formulate the domain-shift rejection problem as an optimization problem in the $\text{H}_\infty$ sense as follows:
%\begin{align}
% &\min_{\mathbf{V}} \;\left\| \;\begin{bmatrix}\mathbf{T}^{\delta d}_z \;\; \mathbf{T}^{\delta s}_z\end{bmatrix} \;\right\|_{\text{H}_\infty} \nonumber\\
% &\text{subject to} \quad \|\boldsymbol{\delta}_{t}\|_2 < \infty: \forall t.
%\label{obj_deviation}   
%\end{align}
\begin{align}
\min_{\mathbf{V},\,\gamma>0}\quad & \gamma \nonumber \\
\text{subject to}\quad 
& \left\| \;\begin{bmatrix}\mathbf{T}^{\delta d}_z \;\; \mathbf{T}^{\delta s}_z\end{bmatrix} \;\right\|_{\text{H}_\infty} \;<\; \gamma, \nonumber\\
& \|\boldsymbol{\delta}_{t}\|_2 < \infty,\; \forall t. \label{obj_deviation}
\end{align}

The optimization problem in eq.~\eqref{obj_deviation} implies that the matrix $\mathbf{V}$ should be designed so that the induced $\text{H}_\infty$ norm from the domain-shift input $\mathbf{d}_z$ and the unintended signal in the state-feedback path $\mathbf{s}_z$ to the state deviation $\boldsymbol{\delta}_z$ is minimized; equivalently,
from the stacked input $\big[\mathbf{d}_z^\top\;\mathbf{s}_z^\top\big]^\top$ to $\boldsymbol{\delta}_z$. In addition, the constraint $\|\boldsymbol{\delta}_{t}\|_2 < \infty$ ensures that the deviation $\boldsymbol{\delta}_t$ remains bounded and small despite the presence of domain shifts, thereby enforcing the stability of the system model described in eq.~(\ref{state_deviation}). Stability here means that states $\boldsymbol{\delta}_t$ of the system (\ref{state_deviation}) do not grow without bound over time.

\subsection{LMI-Constrained Optimization}

Considering the deviation dynamics given in eq.~(\ref{state_deviation}):
\begin{align*}
\boldsymbol{\delta}_t
&= (\mathbf{A} + \mathbf{B}\mathbf{V})\,\boldsymbol{\delta}_{t-1}
   + \mathbf{B}\mathbf{V}\,\mathbf{s}_{t-1}
   + \mathbf{B}\mathbf{d}_t,
\end{align*}
as a system, the performance output of this system can be defined as:
\begin{equation}
\mathbf{\kappa}_t \triangleq \boldsymbol{\delta}_t
\qquad \Rightarrow \quad
\mathbf{C}^{\delta} = \mathbf{I}.
\end{equation}
%In the optimization problem (\ref{obj_deviation}), the goal is to synthesize a static state-feedback gain $\mathbf{V}$ that attenuates the effect of the disturbance $\mathbf{d}_t$ on $\mathbf{z}_t$ in the $\text{H}_\infty$ sense. However, the term $\mathbf{B}\mathbf{V}\mathbf{s}_{t-1}$ acts as an additional known signal that also influences $\boldsymbol{\delta}_t$ unintentionally.
To ensure robustness against both $\mathbf{d}_t$ and $\mathbf{s}_{t-1}$, we treat them jointly as disturbances:
\begin{equation}
\mathbf{w}_t \triangleq
\begin{bmatrix}
\mathbf{d}_t \\[1mm]
\mathbf{s}_{t-1}
\end{bmatrix},
\qquad
\mathbf{B}_w \triangleq
\big[\, \mathbf{B} \;\; \mathbf{B}\mathbf{V} \,\big].
\end{equation}

To compute $\mathbf{V}$, we reformulate the problem in eq.~(\ref{obj_deviation}) as an
optimization problem with an LMI constraint using the
Bounded Real Lemma \cite{scherer2000linear}. 
This lemma provides a convex condition that guarantees both the internal stability of the closed loop and an
$\text{H}_\infty$ performance bound on the induced gain from $\mathbf{w}_t$ to $\mathbf{\kappa}_t$. To use this Lemma, we first introduce the standard transformations
\begin{equation}
\mathbf{R} \triangleq \mathbf{P}^{-1} \succ 0,
\qquad
\mathbf{K} \triangleq \mathbf{V}\mathbf{R}.
\end{equation}
%so that the closed-loop term $\mathbf{A}+\mathbf{B}\mathbf{V}$ becomes linear in $(\mathbf{R},\mathbf{K})$ as $\mathbf{A}\mathbf{R} + \mathbf{B}\mathbf{K}$.
%\paragraph{LMI formulation.}
Accordingly, the LMI optimization formulation to calculate matrices $\mathbf{R}\succ 0$, $\mathbf{K}$, and scalar $\gamma>0$ is as follows:
\begin{align}
\min_{\mathbf{R}\succ0,\,\mathbf{K},\,\gamma>0}\quad & \gamma \label{eq:opt-Hi}\\[2mm]
\text{subject to}\quad
&\begin{bmatrix}
\mathbf{R}
&
\mathbf{A}\mathbf{R}+\mathbf{B}\mathbf{K}
&
\mathbf{R}
&
\mathbf{0}
\\[1mm]
* &
\mathbf{R}
&
\mathbf{0}
&
\big[\,\mathbf{B}\;\; \mathbf{B}\,\big]
\\[1mm]
* & * &
\mathbf{I}
&
\mathbf{0}
\\[1mm]
* & * & * &
\gamma^2 \mathbf{I}_{n_d+n_{\text{total}}}
\end{bmatrix} \succ 0, \label{eq:lmi-Hi}
\end{align}
where $n_d$ and $n_{\text{total}}$ denote the dimensions of $\mathbf{d}_t$ and $\mathbf{s}_{t-1}$, respectively, the symbol $*$ indicates the transpose of the corresponding upper-right block to ensure matrix symmetry, and $\mathbf{0}$ denotes a zero block of the appropriate dimension.\\
If the LMI \eqref{eq:lmi-Hi} is feasible, the optimal state-feedback gain is obtained as
\begin{equation}
\;\mathbf{V} \;=\; \mathbf{K}\,\mathbf{R}^{-1}.\;
\end{equation}
The feasibility guarantees the internal stability of the closed-loop system and ensures that the induced
$\text{H}_\infty$ norm from the combined disturbance $[\mathbf{d}_t^\top\;\mathbf{s}_{t-1}^\top]^\top$ 
to the performance output $\mathbf{\kappa}_t$ satisfies
\[
\left\| \;\begin{bmatrix}\mathbf{T}^{\delta d}_z \;\; \mathbf{T}^{\delta s}_z\end{bmatrix} \;\right\|_{\text{H}_\infty} \;<\; \gamma.
\]
%Considering the dynamics of the state deviation in (\ref{state_deviation}) as:
%\begin{align}
%\boldsymbol{\delta}_t = \mathbf{A}^c \boldsymbol{\delta}_{t-1} + \mathbf{B}^c \mathbf{s}_{t-1} + \mathbf{B} \mathbf{d}_t,
%\label{state_deviation_}
%\end{align}
%where $\mathbf{A}^c = \mathbf{A} + \mathbf{B} \mathbf{V}$ and $\mathbf{B}^c =\mathbf{B}\mathbf{V}$.
%Assuming \(\| {\mathbf{T}}_z^{d\delta}  \|_{\text{H}_\infty} = \kappa\), the objective (\ref{obj_deviation}), considering dynamics model (\ref{state_deviation_}) can be formulated as the following LMI optimization problem using the Bounded Real Lemma \cite{scherer2000linear}:

%\textbf{Note 7}: 
The constraint \( \mathbf{P} \succ 0 \) implies that \( \mathbf{P} \) is symmetric and positive definite, which enforces the stability of the system model described in eq.~(\ref{state_deviation}). 
Importantly, minimizing the effect of the domain shift on the states, that is, the objective function in (\ref{obj_deviation} and \ref{eq:lmi-Hi}), is only meaningful if the system remains stable.

%\textbf{Note 8}: 
For convex synthesis, we reparameterize the additional known signal that also influences $\boldsymbol{\delta}_t$ as 
\(
\mathbf{o}_t \triangleq \mathbf{V}\mathbf{s}_{t-1}
\),
so that the state sees \(\mathbf{o}_t\) through \(\mathbf{B}\).
Accordingly, in the LMI, we use the augmented input matrix
\[
\mathbf{B}_w \;\triangleq\; \big[\,\mathbf{B}\;\;\mathbf{B}\,\big],
\]
which yields an $\text{H}_\infty$ bound from $[\mathbf{d}_t^\top\;\mathbf{o}_t^\top]^\top$ to $\mathbf{\kappa}_t$.
The bound from the original $\mathbf{s}_{t-1}$ satisfies
\(
\|\mathbf{T}^{\kappa s}\|_{\text{H}_\infty}
\le \gamma\,\|\mathbf{V}\|_2.
\)
%Therefore, the constraint \( \mathbf{P} \succ 0 \) is essential in the optimization problem (\ref{lmi_op_}) to ensure that the performance bound is valid under stable dynamics. 

%The decision variables of the optimization problem in (\ref{lmi_op_}) are \( \kappa \), \( \mathbf{P} \), and \( \mathbf{E} \). 
This convex optimization problem (\ref{eq:lmi-Hi}) can be efficiently solved using tools such as CVX (MATLAB), YALMIP (MATLAB), or CVXPY (Python). Once solved, the modifying matrix can typically be computed as \( \mathbf{V} = \mathbf{K} \mathbf{R}^{-1} \). 

\section{Experiments}

In this section, we evaluate the proposed generalization analysis and domain-generalization method, both based on the proposed interpretable model and on a temporal pattern-learning task. Specifically, we use a deep LSTM to predict the future electricity load (i.e. load forecasting), which is crucial to the operation, planning, and economics of modern power systems. The core purpose of short-term (hours to days ahead) load forecasting is to maintain the real-time balance between the supply and the demand. 
%Accurate forecasts support frequency control, blackout prevention, energy trading, and resource planning. 
For training the deep LSTM model and evaluating our proposed approaches, we employ a real, open-source dataset exported from the open energy information (OpenEI) ecosystem operated by national renewable energy laboratory and backed by U.S. department of energy \cite{oedi_portal}. Here, we begin by describing the trained deep LSTM architecture used in our study. Our experimental setup is designed to validate the theory rather than to introduce the state-of-the-art neural network architectures. Next, we present simulation results that evaluate the generalization bound implied by \textbf{Corollary~1} for the trained LSTM and compare it with the OOD generalization bound obtained in \cite{weber2022certifying}, which applies to DNNs. Finally, we use the proposed domain-generalization method to improve the performance of the trained LSTM under the domain shift, thereby reducing its OOD generalization error under the input-distribution shifts.

%%%%%%%%%%
\subsection{Experimental Setups}

\noindent
\textbf{Deep LSTM architecture.} 
We employ a stacked LSTM with two recurrent layers ($N=2$) and hidden size $n_{\text{hid}}=128$ per layer, followed by a linear projection head.
The projection $g(\cdot)$ is a linear layer (i.e. a fully connected map with no activation).
The single input feature is the facility-level electricity consumption (hourly readings, in kW), so the input size is $n_{\text{in}}=1$.
We use a sequence length of $96$, a prediction length of $24$, and a mini-batch size of $64$, as in the training script.
The model's parameters are learned by minimizing the mean-squared error (MSE) between the predicted and ground-truth future sequences (i.e. the loss function $\mathcal{L}$ is MSE).
For each of the $24$ forecast hours, we take the last $24$ hidden states of the top LSTM layer and apply the linear head time-wise, yielding the $24$ outputs:
\(
\{{q}_{t+1},\ldots,{q}_{t+24}\}
\;=\;
\{g(\mathbf{h}^N_{t-23}),\ldots,g(\mathbf{h}^N_{t})\}.
\)
At each time $t$, the model consumes a context window of the most recent $96$ samples and predicts the next $24$ hours:
\[
\mathbf{x}_{t-96+1:t}\ \mapsto\ {\mathbf{q}}_{t+1:t+24}.
\] 

\noindent
\textbf{Dataset}. We use a real dataset provided by OpenEI \cite{oedi_portal}. The dataset provides one full year of hourly California building-energy measurements, from 2018-01-01 01:00 to 2019-01-01 00:00, totaling 8{,}760 time-stamped records with no missing timestamps.
Each record contains 12 metered signals reported as hourly averages in kW, including {Electricity: Facility}, {Gas: Facility}, {Heating: Gas}, {Cooling: Electricity}, and multiple end-use electricity streams. %(e.g., HVAC fans, interior/exterior lighting, interior equipment).
In our forecasting experiments, we focus on the facility-level electricity series ({Electricity: Facility [kW] (Hourly)}) as the target. The data are split chronologically into training/validation/test with a $70\%/20\%/10\%$ ratio. Therefore, over one year (8,760 hours), the split is: training 6,132 hours (8.5 months), validation 1,752 hours (2.4 months), and test 876 hours (1.2 months).

\noindent
\textbf{Interpretable model for the trained deep LSTM.} To construct the interpretable model, we extract the hidden–state trajectories from the trained two-layer LSTM on the test set (876 hours, $\approx$1.2 months), yielding the time-aligned sequences of inputs and states for each layer. For the first LSTM layer, we collect current hidden states, next-step hidden states, and the corresponding inputs across all test samples. 
Similarly, for the second LSTM layer, we collect its hidden states, using the first layer’s outputs as inputs. 
We then apply the DMDc algorithm to each layer separately and estimate linear state-space matrices $\mathbf{A}^{(n)}$ and $\mathbf{B}^{(n)}$ for $n\in\{1,2\}$. 
This yields interpretable linear dynamical systems that approximate the hidden-state evolution. 
Finally, we assemble the full matrices $\mathbf{A}$ and $\mathbf{B}$ according to eq.~\eqref{full_matrix}.
%We validated the quality of these linear approximations by computing reconstruction errors and analyzing the eigenvalue spectra to assess system stability and dominant temporal modes \red{Figures}.

\vspace{0.2cm}
\noindent\textbf{Generate domain shift.} Under domain shift, we generate testing data set with different domain by adding disturbance vector. Let $\mathcal{D} = \{(X_i, Y_i)\}_{i=1}^{N}$ denote the original test dataset, 
where $X_i \in \mathbb{R}^{T}$ represents the input sequence of length 
$T = 96$ hours, $Y_i \in \mathbb{R}^{M}$ represents the prediction 
target of length $M = 24$ hours, and $N$ denotes the number of sequences in the dataset. Under domain shift, we create a new dataset 
$\mathcal{D}' = \{(X'_i, Y'_i)\}_{i=1}^{N}$ by adding bounded disturbances $\mathbf{d}$ to the entire time series, which is constrained by the $\ell_\infty$-norm ($\|\mathbf d\|_\infty 
= \max_{t} |d_{t}|$):
\(
\|\mathbf d\|_{\ell_\infty} \le \alpha,
\)
where $\alpha$ represents the maximum element-wise disturbance magnitude, computed as
\(
\alpha = \lambda \cdot \sigma_{\text{data}},
\)
where $\lambda$ denotes the domain shift strength multiplier and 
$\sigma_{\text{data}}$ the standard deviation of the electricity demand data.
%\textit{Covariate Shift}: For covariate shift, the same disturbance vector $\mathbf d$ is applied to the entire time series, ensuring that
%\[
%\mathbf X'_i = \mathbf X_i + \mathbf d, \quad \mathbf Y'_i = \mathbf Y_i + \mathbf d,
%\]
%where $\mathbf d \in \mathbb{R}^{(T+M)}$ is generated once and applied consistently across both input and target regions. This preserves the temporal dependencies and conditional relationships inherent in the data.

\noindent\textbf{Note:} Our theoretical bounds are stated in terms of $\|\mathbf d\|_{\ell_2}$, whereas in simulations we generate the disturbances by enforcing an element-wise magnitude constraint $\|\mathbf d\|_{\ell_\infty}\le \alpha$. For a length-$T$ sequence, these norms satisfy
\begin{equation*}
\|\mathbf d\|_{\ell_2} \le \sqrt{T}\,\|\mathbf d\|_{\ell_\infty} \le \sqrt{T}\,\alpha,
\end{equation*}
so, when applying the $\ell_2$-based GE bounds to our $\ell_\infty$-bounded disturbances, we substitute $\|\mathbf d\|_{\ell_2}$ by $\sqrt{T}\alpha$, which introduces the $\sqrt{T}$ factor in the reported bound.

We implement three variants of the domain shift to model different realistic scenarios:

\paragraph{ Uniform domain shift}
It applies random disturbances uniformly sampled. This represents mild, stochastic variations in the data distribution, such as measurement noise or minor sensor calibration drift. Each element of the disturbance is independently sampled from a uniform distribution, resulting in an average perturbation magnitude of approximately $\alpha/2$.

\paragraph{Fixed magnitude domain shift}
It represents the worst-case adversarial scenario where disturbances always saturate the $\ell_\infty$ bound. The sign of each disturbance is randomly chosen, but the magnitude is fixed at $\alpha$, ensuring \(
\|\mathbf d\|_\infty = \alpha.
\)
This enables testing the maximum generalization of the model under the strongest possible bounded shift, e.g., adversarial conditions or extreme sensor malfunction.
\paragraph{Mixture domain shift}
It combines the small and large shifts to the dataset:
\[
 d_{t} = 
\begin{cases}
\mathcal{U}(-0.1\alpha, 0.1\alpha), & \text{with probability } 0.7,\\[2pt]
\mathcal{U}(-\alpha, \alpha), & \text{with probability } 0.3.
\end{cases}
\]
After sampling, temporal smoothing is applied using a Gaussian filter with $\sigma = 12$ hours to ensure a realistic temporal continuity. Finally, the perturbation is rescaled if necessary to enforce the $\ell_\infty$ constraint:
\[
\mathbf d_{\text{final}} =
\begin{cases}
\mathbf d', & \text{if } \|\mathbf d'\|_\infty \leq \alpha,\\[2pt]
\mathbf d' \cdot \dfrac{0.95\,\alpha}{\|\mathbf d'\|_\infty}, & \text{otherwise},
\end{cases}
\]
where $\mathbf d'$ denotes the smoothed perturbation. This use case simulates the occasional sensor spikes combined with a typical measurement noise.

Domain shift classification is magnitude-dependent. We distinguish covariate shift from concept shift by comparing the (lag-0) Pearson cross-correlation between inputs and outputs before and after shift. Let $
\rho^{\text{orig}}_{XY}=\mathrm{Corr}(X,Y), 
\qquad
\rho^{\text{shift}}_{XY}=\mathrm{Corr}(X',Y').
$
If the shift is predominantly {covariate} (i.e. $P(X)$ changes while $P(Y\mid X)$ is unchanged), then the input--output coupling should remain stable, yielding
$
\rho^{\text{shift}}_{XY}\approx \rho^{\text{orig}}_{XY}\quad(\text{typically } >0.95).
$ In contrast, under the {concept} shift (i.e. $P(Y\mid X)$ changes), the coupling weakens and the correlation drops,
$
\rho^{\text{shift}}_{XY}\ll \rho^{\text{orig}}_{XY}\quad(\text{typically } <0.5).
$
For multi-step forecasting, we compute the horizon-wise correlations and average across horizons. To avoid spurious effects, we align the forecast horizon (lag-0), and standardize within each regime prior to correlation.
\begin{figure*}[!t]
  \centering
  \subfloat[Empirical bound vs.\ theoretical upper bound]{%
    \includegraphics[width=0.48\textwidth]{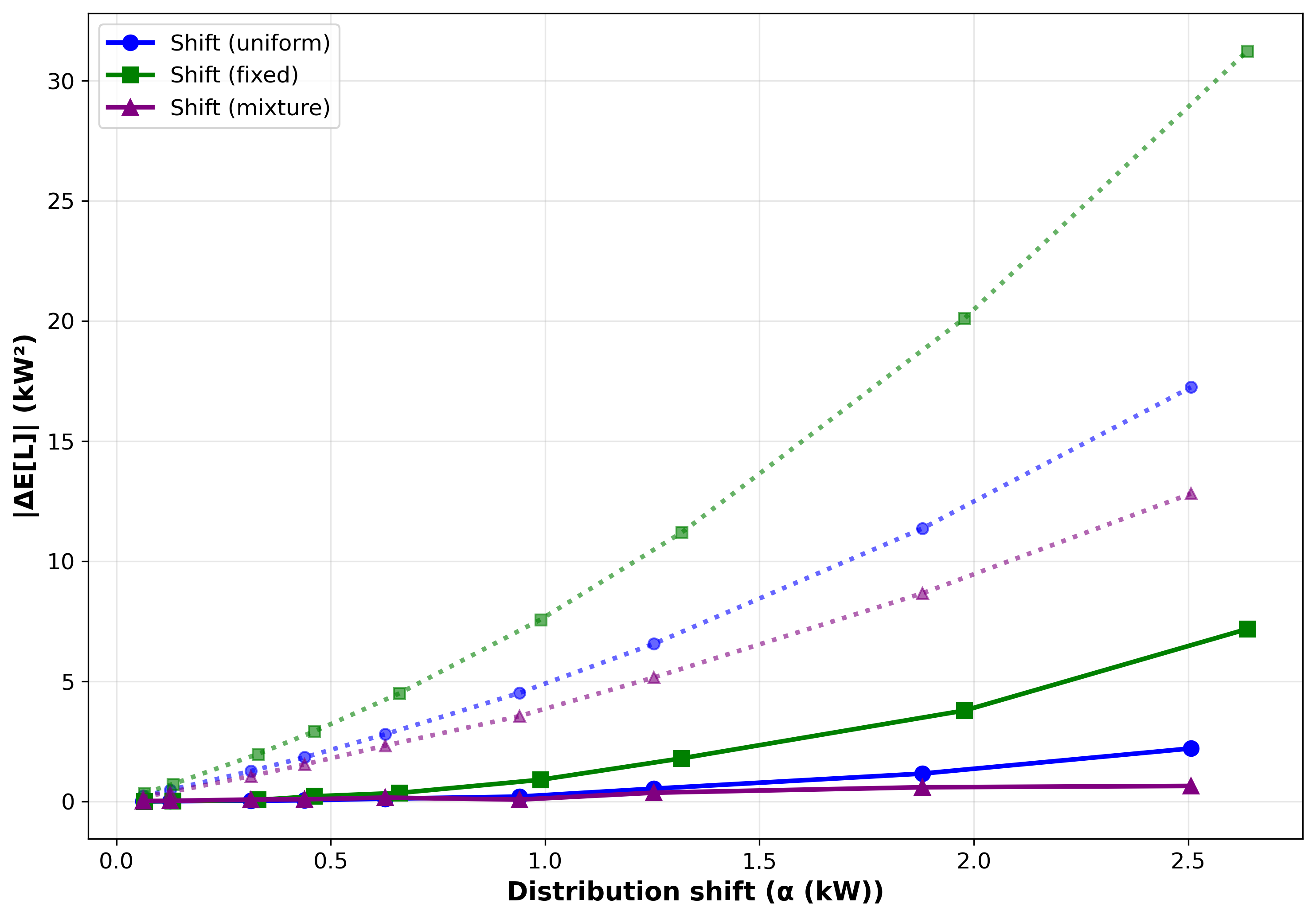}%
    \label{Emperical bound vs theoretical upper bound}}
  \hfil
  \subfloat[RMSE increase (kW)]{%
    \includegraphics[width=0.48\textwidth]{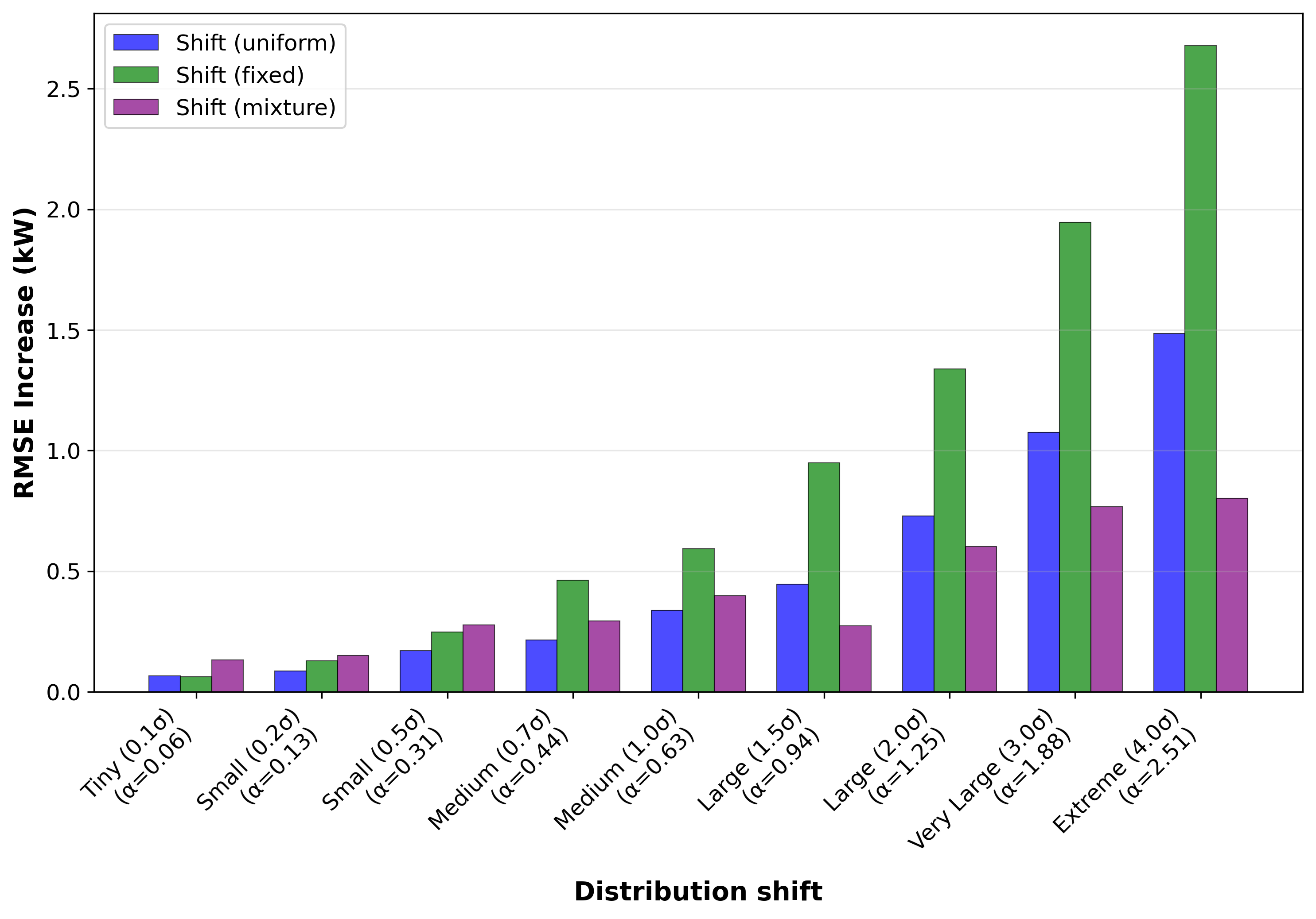}%
    \label{RMSE increase}}
  \caption{Generalization analysis under different domain-shift scenarios: (a) empirical bound versus theoretical upper bound; (b) RMSE increase.}
  \label{fig:rmse-increase}
\end{figure*}

\subsection{Analysis of Generalization Bound }

To validate the proposed generalization bound for the domain shift in \textbf{Corollary~1}, we conduct  experiments under three shift types: uniform, fixed, and mixture domain shifts. Fig.~\ref{Emperical bound vs theoretical upper bound} compares the theoretical upper bound (dotted) with the empirical change in the expected loss (solid) versus domain shift magnitude $\alpha$.
Across all shifts, the bound stays above the empirical curves (never violated) and it is the tightest for the fixed shift and more conservative for uniform/mixture domain shifts.
Fig.~\ref{RMSE increase} reports RMSE increase (kW) across nine scenarios as $\alpha$ grows from $0.08$ (Tiny) to $2.55$ (Extreme). RMSE increase demonstrates the practical impact of domain shift on forecasting accuracy, showing that the prediction errors can increase from negligible levels (0.05 kW) under small shift to substantial levels (2.7 kW) under extreme shifts.

\begin{figure*}[!ht]
\centering
\subfloat[Bounds comparison]{% 
\includegraphics[width=0.48\textwidth]{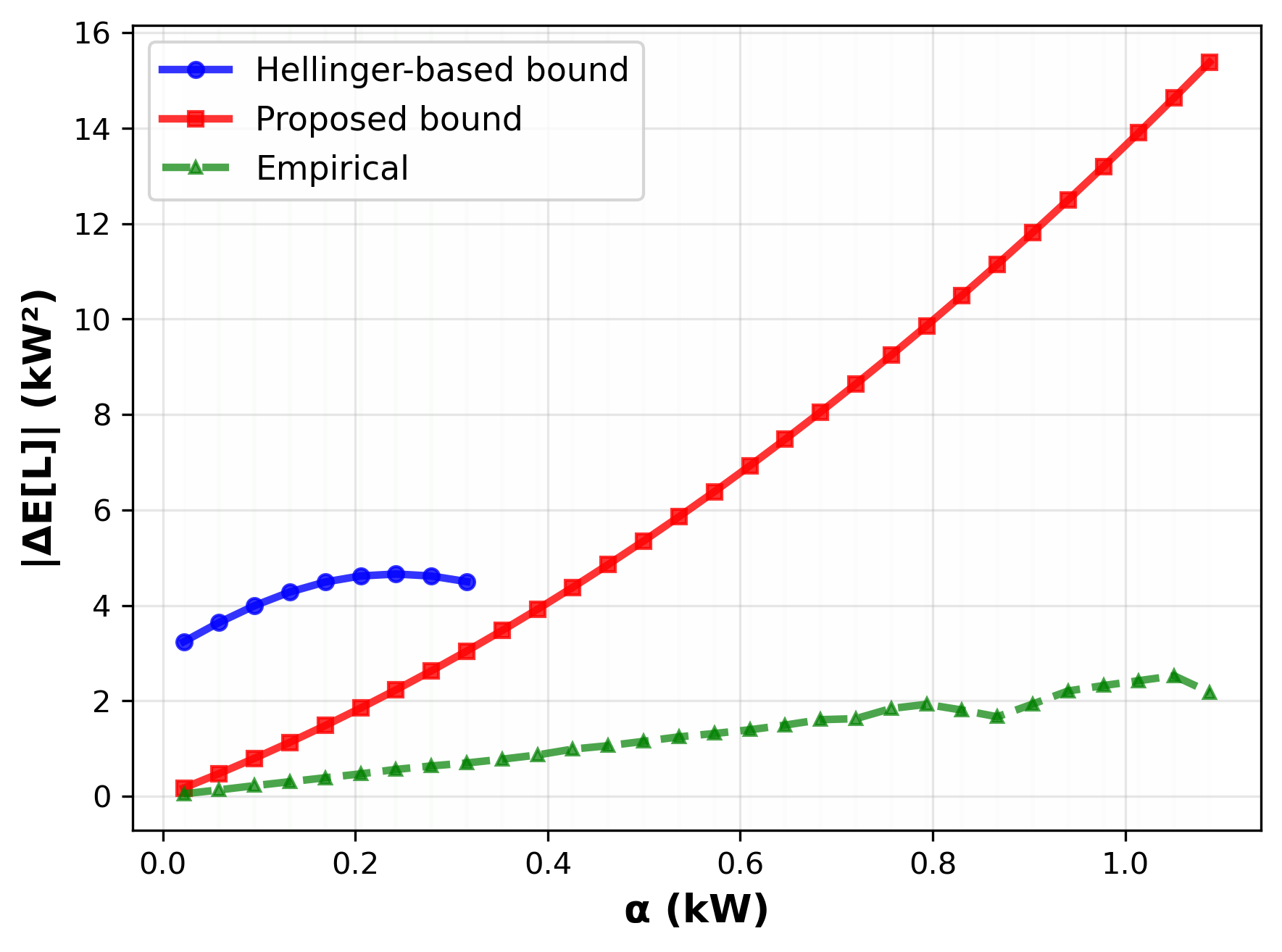}\label{Comparison}}
\hfil \subfloat[Validity regions]{%
\includegraphics[width=0.48\textwidth]{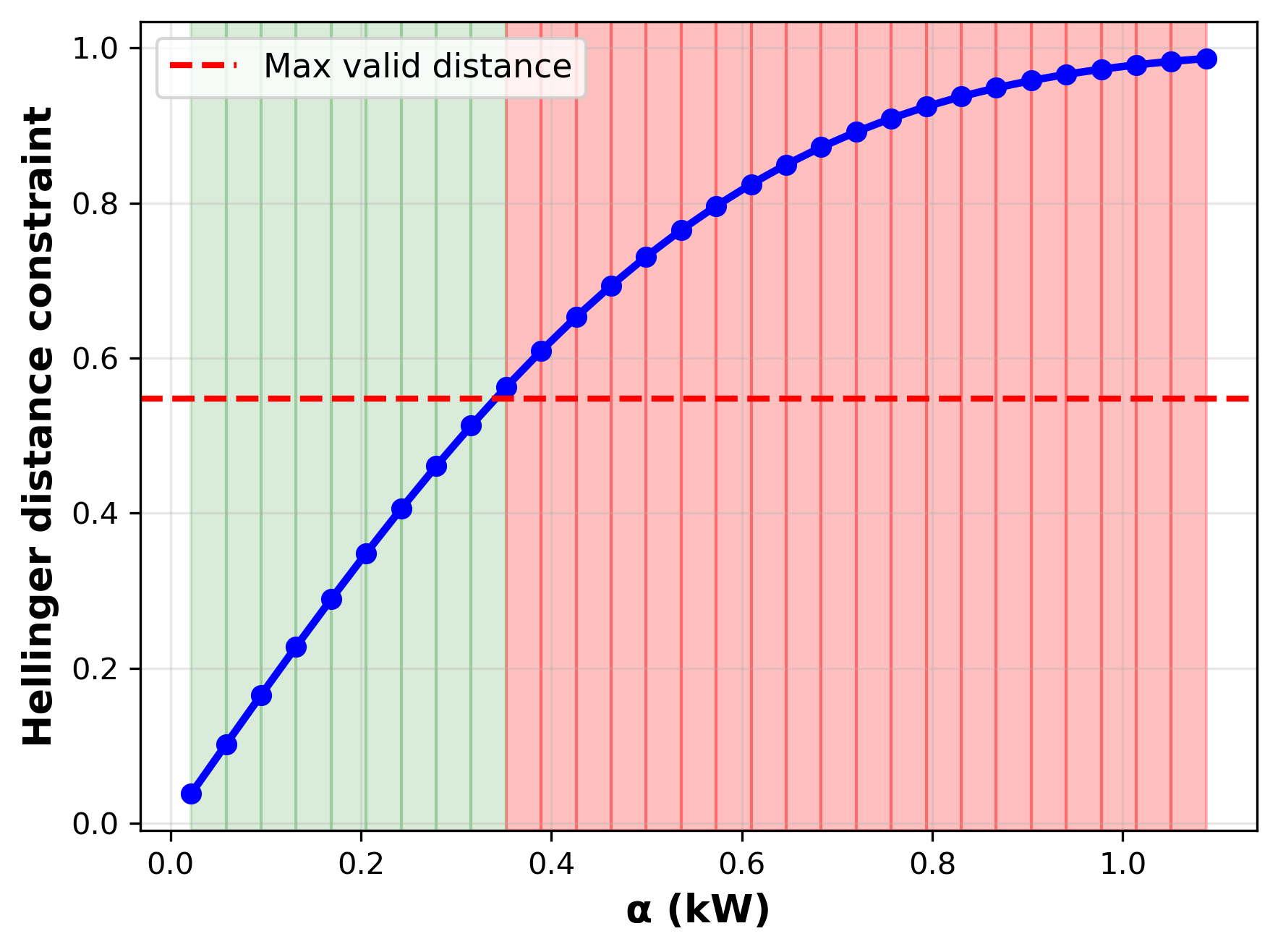}\label{valid_region}}
\caption{(a) Comparison with Hellinger distance-based bound across different domain shift scenarios; (b) validity regions for the Hellinger distance-based bound)}
\label{fig:rmse-and-class-wide_}
\end{figure*}

\vspace{0.2cm}
\noindent \textbf{Comparison with Hellinger distance-based bound.} We conduct a fair comparison with a baseline bound based on the Hellinger distance~\cite{weber2022certifying}. 
As demonstrated in~\cite{weber2022certifying}, this bound can be substantially tighter than the earlier Wasserstein-distance-based certificates, e.g.,~\cite{cranko2021generalised}. %As shown in~\cite{weber2022certifying}, this bound is much tighter than Wasserstein distance-based bounds presented in~\cite{cranko2021generalised, sinha2017certifiable}.
For this reason, we compare our approach with the Hellinger distance-based bound. This baseline method assumes that the data distribution is known and uses the Hellinger distance between the original distribution $\mathcal{D}$ and the shifted distribution $\mathcal{D}'$. To enable this comparison, we first approximate the OpenEI dataset distribution using a Gaussian Mixture Model. To compare with the Hellinger distance-based bound, we generate two types of covariate shifts: a fixed disturbance vector $\mathbf{d}$ and Gaussian shift (are drawn independently per timestep).
%\begin{itemize}
 %   \item \textbf{Deterministic shift:} A fixed disturbance vector $\mathbf{d}$ with 
%    \[
%        \|\mathbf{d}\|_2 = \alpha
%    \]
%    is applied uniformly across all time steps.
%    \item \textbf{Gaussian shift:} Random disturbances 
%    \[
%        \mathbf{d}_t \sim \mathcal{N}(\mathbf{m}_d, \boldsymbol{\Sigma}_d)
%    \]
%    are drawn independently per timestep, satisfying
 %   \[
        %\mathbb{E}\big[\|\mathbf{d}_t\|_2\big] \leq \alpha.   \]
%\end{itemize}
For each shift magnitude $\alpha$, we compute the Hellinger distance and evaluate both bounds.
As shown in Fig.~\ref{Comparison}, our bound is significantly tighter than the Hellinger-based bound within its valid region.
A critical limitation of the Hellinger approach is its strict validity constraint: the bound can only be applied when the Hellinger distance satisfies the strict validity constraint which, restricts its applicability to small domain shifts. Beyond a problem-specific threshold (red line), the Hellinger-based bound diverges to infinity and becomes unusable, providing no guarantee under moderate-to-large shifts (see Fig.~\ref{valid_region}).
In our electricity-forecasting setting, this threshold occurs at roughly $0.35$~kW, so the competing method is usable only for small shifts. By contrast, our bound remains valid and finite across the entire range of shift magnitudes that are tested.
\begin{figure}[!ht]
    \centering
    \includegraphics[width=\columnwidth]{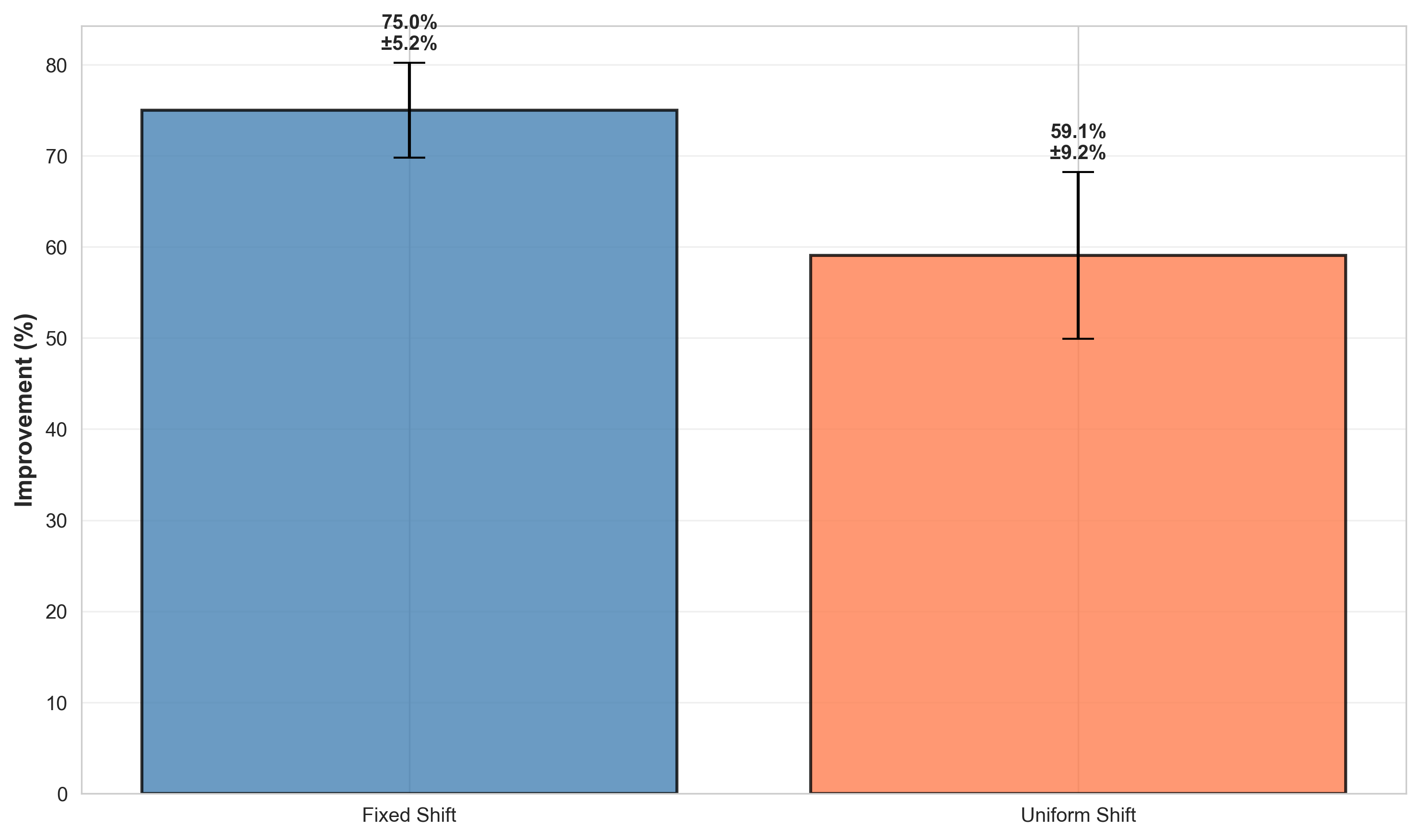}
  \caption{DG method improvement comparison: Fixed vs. uniform domain shift}
  \label{Improvement}
\end{figure}
\vspace{-10pt}

\subsection{Performance of the Domain Generalization Method}
We evaluate the performance of the domain generalization method proposed in Section~\ref{OOD_method}. We generate domain shift scenarios under \textbf{Assumption 9}. Specifically, we generate domain shift on the OpenEI dataset under two scenarios: fixed and uniform domain shift, and conduct 10 independent trials for each shift type to ensure the statistical robustness.

\noindent \textbf{Controllability-based model reduction.}
The interpretable model of the trained LSTM has a state matrix $\mathbf{A}\in\mathbb{R}^{512\times512}$, but not all states are controllable via the input feedback. In practice, these weakly controllable states often have limited influence on the output and mainly increase the computational cost of solving optimization problem~\eqref{eq:lmi-Hi}. To address this, we apply a controllability-Gramian-based decomposition to the controllable subspace. The details of this approach are provided in \textbf{Appendix~V}.
%presents significant computational challenges for H$_\infty$ controller synthesis. More critically, not all LSTM states are controllable via input feedback. To address this, we employ a controllability Gramian-based decomposition to identify controllable states. 
Fig.~\ref{Improvement} shows the improvement across trials, confirming a robust performance. These results validate our DG method, yielding average MSE reductions of 75\% (fixed shift) and 59\% (uniform shift), and significantly outperforming the baseline without domain generalization. Additional simulation results are provided in \textbf{Appendix VI}.
\section{Conclusion}
We have developed an interpretable representation of RNNs, with a focus on LSTMs. We model the learned LSTM dynamics as an unknown discrete-time nonlinear feedback system and used the Koopman operator theory with DMDc to obtain a data-driven linear approximation. By modeling the domain shifts as additive input disturbances, we have employed the $\text{H}_{\infty}$ norm to quantify the worst-case state deviations and to characterize the generalization behavior under both covariate shift and concept drift. Building on this analysis, we have proposed a domain generalization approach that improves the robustness to input distribution shifts by augmenting the trained LSTM with a domain-shift rejection layer, designed via an LMI-constrained optimization. Future work will extend the proposed framework to broader classes of DNN architectures.We will also investigate a domain-shift detection stage and explore how to couple it with the proposed DG procedure to move toward a complete pipeline. %develop online/adaptive variants that update the DG method in real time as the operating domain evolves.

\section*{Appendices}
\subsection{Numerical Procedure for Computing $\mathbf{A}$ and $\mathbf{B}$ with the Identity Observable}

These matrices $\mathbf{A}$ and $\mathbf{B}$ are given in Section III C.

Given the data matrices
\(
\mathbf{S}=\big[\mathbf{s}_{0},\ldots,\mathbf{s}_{\bar{T}-1}\big],\quad
\mathbf{\Xi}=\big[\mathbf{s}_{1},\ldots,\mathbf{s}_{\bar{T}}\big],\quad
\mathbf{X}=\big[\mathbf{x}_{1},\ldots,\mathbf{x}_{\bar{T}}\big],
\) form the stacked snapshot matrix
\begin{equation*}
\mathbf{\Omega} \;:=\;
\begin{bmatrix}
\mathbf{S}\\[2pt]\mathbf{X}
\end{bmatrix}
\in \mathbb{R}^{(n_{\text{state}}+n_{\text{in}})\times \bar{T}} .
\label{eq:Omega_stack}    
\end{equation*}
The numerical procedure of $\mathbf{A}$ and $\mathbf{B}$ is:

\noindent
%\paragraph{Low-rank bases:}
i.  {\em Low-rank bases}:
Compute the truncated SVDs
\begin{equation*}
\mathbf{\Omega} \approx \tilde{\mathbf{U}}\,\tilde{\mathbf{\Sigma}}\,\tilde{\mathbf{V}}^{\!*},\qquad
\mathbf{\Xi} \approx \hat{\mathbf{U}}\,\hat{\mathbf{\Sigma}}\,\hat{\mathbf{V}}^{\!*},
\label{eq:twoSVDs}
\end{equation*}
with truncation ranks \(p\) and \(r\), respectively. Partition \(\tilde{\mathbf{U}}\) compatibly with $\mathbf{\Omega}$:
\[
\tilde{\mathbf{U}}=
\begin{bmatrix}
\tilde{\mathbf{U}}_{1}\\ \tilde{\mathbf{U}}_{2}
\end{bmatrix},\quad
\tilde{\mathbf{U}}_{1}\in\mathbb{R}^{n_{\text{state}}\times p},\;\;
\tilde{\mathbf{U}}_{2}\in\mathbb{R}^{n_{\text{in}}\times p}.
\]

\noindent
%\paragraph{Projected (reduced) operators:}
ii. {\em Projected (reduced) operators}:
Projecting onto the output basis \(\hat{\mathbf{U}}\) yields reduced-order operators
\begin{equation*}
\tilde{\mathbf{A}} \;=\; \hat{\mathbf{U}}^{\!*}\,\mathbf{\Xi}\,\tilde{\mathbf{V}}\,
\tilde{\mathbf{\Sigma}}^{-1}\,\tilde{\mathbf{U}}_{1}^{\!*}\,\hat{\mathbf{U}},
\qquad
\tilde{\mathbf{B}} \;=\; \hat{\mathbf{U}}^{\!*}\,\mathbf{\Xi}\,\tilde{\mathbf{V}}\,
\tilde{\mathbf{\Sigma}}^{-1}\,\tilde{\mathbf{U}}_{2}^{\!*}.
\label{eq:AB_tilde}
\end{equation*}
In the reduced coordinates \(\mathbf{z}_{t}:=\hat{\mathbf{U}}^{\!*}\mathbf{s}_{t}\in\mathbb{R}^{r}\),
\begin{equation*}
\mathbf{z}_{t} \;\approx\; \tilde{\mathbf{A}}\,\mathbf{z}_{t-1} + \tilde{\mathbf{B}}\,\mathbf{x}_{t},
\qquad
\mathbf{s}_{t} \;\approx\; \hat{\mathbf{U}}\,\mathbf{z}_{t}.
\label{eq:reduced_dmdc}
\end{equation*}

\noindent
%\paragraph{Spectral form and dynamic modes:}
iii. {\em Spectral form and dynamic modes}:
Compute the eigendecomposition
\begin{equation*}
\tilde{\mathbf{A}}\,\mathbf{W} = \mathbf{W}\,\mathbf{\Lambda},
\label{eq:eigAtilde}
\end{equation*}
and define the DMDc modes (in the full state space)
\begin{equation*}
\mathbf{\Phi} \;=\; \mathbf{\Xi}\,\tilde{\mathbf{V}}\,\tilde{\mathbf{\Sigma}}^{-1}\,
\tilde{\mathbf{U}}_{1}^{\!*}\,\hat{\mathbf{U}}\,\mathbf{W}.
\label{eq:modes}
\end{equation*}
In modal coordinates \(\mathbf{y}_{t}:=\mathbf{W}^{-1}\mathbf{z}_{t}\),
\begin{equation*}
\mathbf{y}_{t} \;\approx\; \mathbf{\Lambda}\,\mathbf{y}_{t-1}
+ \underbrace{\mathbf{W}^{-1}\tilde{\mathbf{B}}}_{\tilde{\mathbf{B}}_{\text{modal}}}\,\mathbf{x}_{t},
\qquad
\mathbf{s}_{t} \;\approx\; \hat{\mathbf{U}}\,\mathbf{W}\,\mathbf{y}_{t}.
\label{eq:modal_dmdc}
\end{equation*}

\noindent
%\paragraph{Back to full-state operators:}
iv. {\em Back to full-state operators}:
A full-state pair can be recovered via
\begin{equation}
\mathbf{A}_{\text{proj}} := \hat{\mathbf{U}}\,\tilde{\mathbf{A}}\,\hat{\mathbf{U}}^{\!*},\qquad
\mathbf{B}_{\text{proj}} := \hat{\mathbf{U}}\,\tilde{\mathbf{B}}.
\label{eq:project_back}
\end{equation}
When no truncation is applied (i.e., \(r=n_{\text{state}}\) and \(p=n_{\text{state}}+n_{\text{in}}\)) and the SVDs are exact, $\mathbf{A}_{\text{proj}}$ and $\mathbf{B}_{\text{proj}}$ coincide with the standard least-squares solution:
\[
\big[\,\mathbf{A}\;\;\mathbf{B}\,\big]
\;=\;
\mathbf{\Xi}\,
\begin{bmatrix}
\mathbf{S}\\ \mathbf{X}
\end{bmatrix}^{\!\dagger},
\]
which is the identity-observable instance of the general estimator in (9). %\eqref{opt_}.
Note that the truncation ranks \(p\) and \(r\) control model order (choose via cumulative energy of singular values or cross-validation). 

\subsection{Proof of Theorem~1}

\begin{proof}
Since $\mathbf T^{dh}(z)$ is stable, for any disturbance sequence $\{\mathbf d_t\}_{t=0}^T$:
\begin{equation*}
\|\Delta \mathbf h^N\|_{\ell_2}
\;\le\; \|\mathbf T^{dh}_z\|_{ H_\infty}\,\|\mathbf d\|_{\ell_2}
\;\le\; \|\mathbf T^{dh}_z\|_{ H_\infty}\,\alpha.
\end{equation*}
Squaring both sides yields
\begin{equation*}
\sum_{t=0}^T \|\Delta \mathbf h^N_t\|_2^2
=\|\Delta \mathbf h^N\|_{\ell_2}^2
\;\le\; \|\mathbf T^{dh}_z\|_{ H_\infty}^{\,2}\,\alpha^2,
\end{equation*}
which proves (35).
Finally, as each summand is nonnegative,
\begin{align*}
\max_{0\le t\le T}\big\|\hat{\mathbf h}^{\,N}_t - \mathbf h^{\,N}_t\big\|_2\;&\le\; \Big(\sum_{t=0}^{T}\|\hat{\mathbf h}^{\,N}_t-\mathbf h^{\,N}_t\|_2^2\Big)^{\!1/2}\\
\;&\le\; \|\mathbf T^{dh}_z\|_{H_\infty}\,\alpha,
\end{align*}
which proves (36).
\end{proof}

\subsection{Proof of Corollary 1}

\begin{proof}
According to \textbf{Theorem 1}, for any domain shift $\{\mathbf d_t\}_{t=0}^T$ such that \( \|{\mathbf{d}}\|_{\ell_2} \leq \alpha \), we have
\[\max \|\hat{{\mathbf{h}}}^N_t-{\mathbf{h}}^N_t\|_2 \leq \|{\mathbf{T}}_z^{dh}\|_{H_\infty} \cdot \alpha, \quad \forall t \in \{0,1,...,T\}.\]
In addition, based on \textbf{Assumption 8}, we can write
\begin{align*} 
|\mathcal{L}(\hat{\mathbf{y}}_t, g(\hat{{\mathbf{h}}}^N_{t-\tau:t}) &- \mathcal{L}(\mathbf{y}_t, g({\mathbf{h}}^N_{t-\tau:t}))| \leq L \cdot (\|\hat{\mathbf{y}}_t - \mathbf{y}_t\|_2\\&+ 
\|g(\hat{{\mathbf{h}}}^N_{t-\tau:t}) - g({\mathbf{h}}^N_{t-\tau:t}) \|_2).
\end{align*}
Then, following  \textbf{Assumption 7}, we have
\begin{align*}
|\mathcal{L}(\hat{\mathbf{y}}_t, g(\hat{{\mathbf{h}}}^N_{t-\tau:t}) &- \mathcal{L}(\mathbf{y}_t, g({\mathbf{h}}^N_{t-\tau:t}))| \leq  L \cdot (\beta \\ &+ 
G \cdot \sqrt{\tau+1} \cdot \|{\mathbf{T}}_z^{dh}\|_{H_\infty} \cdot \alpha).
\end{align*}
Therefore, we can derive
\begin{align*}
|\mathcal{L}(\hat{\mathbf{y}}_t, \hat{{\mathbf{q}}}_t) &- \mathcal{L}(\mathbf{y}_t, {\mathbf{q}}_t)| \leq  L \cdot (\beta\\&+%\|\hat{\mathbf{y}}_t - \mathbf{y}_t\|_2 \\ &+ 
G \cdot\sqrt{\tau+1} \cdot  \|{\mathbf{T}}_z^{dh}\|_{H_\infty} \cdot \alpha), \; \forall t \in \{0,1,...,T\}.
\end{align*}
Then, taking the expectation gives
\begin{align*}
\mathbb{E}\left[\left|\mathcal{L}(\hat{\mathbf{y}}_t, \hat{{\mathbf{q}}}_t) - \mathcal{L}(\mathbf{y}_t, {\mathbf{q}}_t) \right|\right] &\leq  L \cdot (\beta \\
&+G \cdot \sqrt{\tau+1} \cdot\|{\mathbf{T}}_z^{dh}\|_{H_\infty} \cdot \alpha).
\end{align*}
Using Jensen's inequality and the convexity of the norm
\begin{align*}
\left| \mathbb{E}\left[\mathcal{L}(\mathbf{y}_t, \hat{{\mathbf{q}}}_t)\right] - \mathbb{E}\left[\mathcal{L} (\mathbf{y}_t, {\mathbf{q}}_t)\right] \right|&\leq  L \cdot (\beta \\&+G \cdot \sqrt{\tau+1} \cdot\|{\mathbf{T}}_z^{dh}\|_{H_\infty} \cdot \alpha).
\end{align*}\end{proof}
\section{Extending Corollary~1 to the MSE Loss}

\begin{proof}
Assume $\mathbf{y}_t$ and $\mathbf{q}_t$ are real-valued and
$\mathcal{L}(\mathbf y_t,\mathbf q_t)=\|\mathbf y_t-\mathbf q_t\|_2^2$.
Define $\mathbf e_t := \mathbf y_t-\mathbf q_t$,
$\Delta \mathbf y_t := \hat{\mathbf y}_t-\mathbf y_t$, and
$\Delta \mathbf q_t := \hat{\mathbf q}_t-\mathbf q_t$.
Then $\hat{\mathbf e}_t=\hat{\mathbf y}_t-\hat{\mathbf q}_t=\mathbf e_t+(\Delta \mathbf y_t-\Delta \mathbf q_t)$, and
\begin{align*}
\mathcal{L}(\hat{\mathbf y}_t,\hat{\mathbf q}_t)-\mathcal{L}(\mathbf y_t,\mathbf q_t)
&= \|\hat{\mathbf e}_t\|_2^2-\|\mathbf e_t\|_2^2 \\
&= (\hat{\mathbf e}_t-\mathbf e_t)^\top(\hat{\mathbf e}_t+\mathbf e_t) \\
&= (\Delta \mathbf y_t-\Delta \mathbf q_t)^\top\!\left(2\mathbf e_t+(\Delta \mathbf y_t-\Delta \mathbf q_t)\right).
\end{align*}
By Cauchy--Schwarz inequality,
\begin{align*}
\big|\mathcal{L}(\hat{\mathbf y}_t,\hat{\mathbf q}_t)-\mathcal{L}(\mathbf y_t,\mathbf q_t)\big|
&\le 2\|\mathbf e_t\|_2\,\|\Delta \mathbf y_t-\Delta \mathbf q_t\|_2\\
&+\|\Delta \mathbf y_t-\Delta \mathbf q_t\|_2^2.
\end{align*}
Using the triangle inequality,
\[\|\Delta \mathbf y_t-\Delta \mathbf q_t\|_2\le \|\Delta \mathbf y_t\|_2+\|\Delta \mathbf q_t\|_2.\]
If $\|\Delta \mathbf y_t\|_2\le \beta$ and
$\|\Delta \mathbf q_t\|_2\le G\sqrt{\tau+1}\,\|{\mathbf{T}}_z^{dh}\|_{{H}_\infty}\alpha$, then
\begin{align*}
\big|\mathcal{L}(\hat{\mathbf y}_t,\hat{\mathbf q}_t)-\mathcal{L}(\mathbf y_t,\mathbf q_t)\big|
&\le 2\|\mathbf e_t\|_2\Big(\beta+G\sqrt{\tau+1}\,\|{\mathbf{T}}_z^{dh}\|_{{H}_\infty}\alpha\Big)\\
&+\Big(\beta+G\sqrt{\tau+1}\,\| {\mathbf{T}}_z^{dh}\|_{{H}_\infty}\alpha\Big)^2.
\end{align*}
Using $\|\mathbf e_t\|_2=\sqrt{\mathcal{L}(\mathbf y_t,\mathbf q_t)}$ and taking expectation on both sides yields
\begin{align*}
\mathbb{E}\big|\mathcal{L}(\hat{\mathbf y}_t,\hat{\mathbf q}_t)-\mathcal{L}(\mathbf y_t,\mathbf q_t)\big|
&\le 2\,\mathbb{E}\sqrt{\mathcal{L}(\mathbf y_t,\mathbf q_t)}
\Big(\beta\\&+G\sqrt{\tau+1}\,\|{\mathbf{T}}_z^{dh}\|_{{H}_\infty}\alpha\Big)\\
&\quad+\Big(\beta+G\sqrt{\tau+1}\,\|{\mathbf{T}}_z^{dh}\|_{{H}_\infty}\alpha\Big)^2.
\end{align*}
Moreover,
\begin{align*}
\big|\mathbb{E}\mathcal{L}(\hat{\mathbf y}_t,\hat{\mathbf q}_t)-\mathbb{E}\mathcal{L}(\mathbf y_t,\mathbf q_t)\big|
&=\big|\mathbb{E}\big[\mathcal{L}(\hat{\mathbf y}_t,\hat{\mathbf q}_t)-\mathcal{L}(\mathbf y_t,\mathbf q_t)\big]\big|\\
&\le \mathbb{E}\big|\mathcal{L}(\hat{\mathbf y}_t,\hat{\mathbf q}_t)-\mathcal{L}(\mathbf y_t,\mathbf q_t)\big|.
\end{align*}
Finally,
\begin{align*}
    \big|\mathbb E\,\mathcal{L}(\hat{\mathbf y}_t,\hat{\mathbf q}_t)-\mathbb E\,\mathcal{L}(\mathbf y_t,\mathbf q_t)\big|&\le 2\,\mathbb E\,\sqrt{\mathcal{L}(\mathbf y_t,\mathbf q_t)}\,(\beta\\&+G \cdot \sqrt{\tau+1} \cdot\|{\mathbf{T}}_z^{dh}\|_{H_\infty} \cdot \alpha)\\&+(\beta+G \cdot \sqrt{\tau+1} \cdot\|{\mathbf{T}}_z^{dh}\|_{H_\infty} \cdot \alpha)^2.
\end{align*}
\end{proof}

\subsection{Controllability-Based Model Reduction}

Our approach for controllability-based model reduction is as follows:
\begin{enumerate}
    \item The controllability Gramian $\mathbf P_c$ satisfies the discrete Lyapunov equation:
    \begin{equation}
    \mathbf P_c = \mathbf A \mathbf P_c \mathbf A^\top + \mathbf B \mathbf B^\top.
    \end{equation}
    
    \item Through eigenvalue decomposition $\mathbf P_c = \mathbf U \mathbf \Sigma \mathbf U^\top$ and selecting eigenvalues above a threshold, we identify $n_c = 50$ truly controllable modes out of 512 total states.
    
    \item The transformation matrix $\mathbf T_c \in \mathbb{R}^{512 \times n_c}$ (formed from the top $n_c$ eigenvectors) projects the full state space onto the controllable subspace.
    
    \item The disturbance input matrix exhibits strong localization: $\|\mathbf B_c\| \gg \|\mathbf B_{uc}\|$, where $\mathbf B_c$ corresponds to controllable states and $\mathbf B_{uc}$ to uncontrollable states. This verifies that domain shifts primarily affect the controllable subspace.
\end{enumerate}
Using this decomposition, we reduce the interpretable model to its controllable subspace and compute the feedback gain for the controllable subsystem, denoted by $\mathbf{V}_c$:
\begin{align*}
    \hat{\mathbf {s}}_{c,t} &= \mathbf A_{cc} \hat{\mathbf s}_{c,t-1} + \mathbf B_c \mathbf x_t + \mathbf B_c \mathbf d_t, \nonumber \\
    \hat{\mathbf h}^N_t &= \mathbf C_c \hat{\mathbf s}_{c,t},
\end{align*}
where $\mathbf A_{cc} \in \mathbb{R}^{n_c \times n_c}$, $\mathbf B_c \in \mathbb{R}^{n_c \times 1}$, $\mathbf C_c \in \mathbb{R}^{1 \times n_c}$, and $\mathbf d_t$ represents the domain shift disturbance. We then calculate the optimal state-feedback gain $\mathbf{V}_c \in \mathbb{R}^{1 \times n_c}$ using CVXPY with MOSEK/SCS solvers. For implementation, we project the full LSTM state to the controllable subspace via $\mathbf s_{c,t} = \mathbf T_c^\top \mathbf s_t$ and apply the control law:
\[\mathbf f_t = \mathbf{V}_c \mathbf x_{c,t}.\]

\subsection{Additional Simulation Results }

\paragraph{Analysis of Generalization Bound}

Fig.~\ref{classifies} classifies whether each domain shift scenario results in a covariate shift (COV, shown in green) or concept shift (CON, shown in red). This classification is crucial because it demonstrates the broad applicability of our proposed bound.
\begin{figure}[!h]
  \centering
  \includegraphics[width=0.95\linewidth]{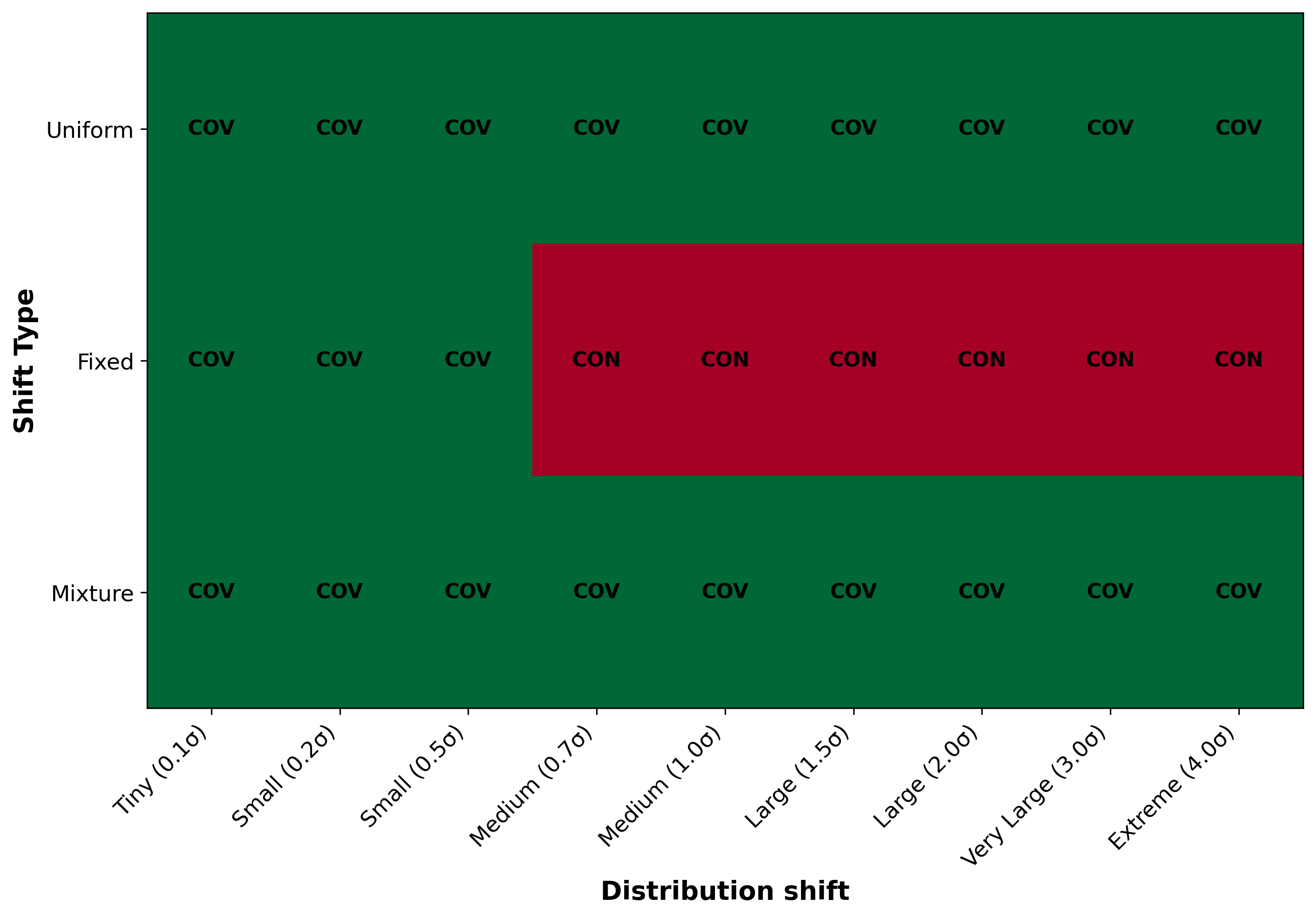}
  \caption{Domain shift classification (covariate shift vs.\ concept shift).}
  \label{classifies}
\end{figure}
\paragraph{Performance of the Domain Generalization Method}

Fig. ~\ref{fig:evolution} illustrates the MSE reduction achieved by the proposed DG method across trials under both fixed and uniform domain shifts.
\begin{figure*}[!t]
  \centering
  \subfloat[MSE evolution—fixed shift]{%
    \includegraphics[width=0.48\textwidth]{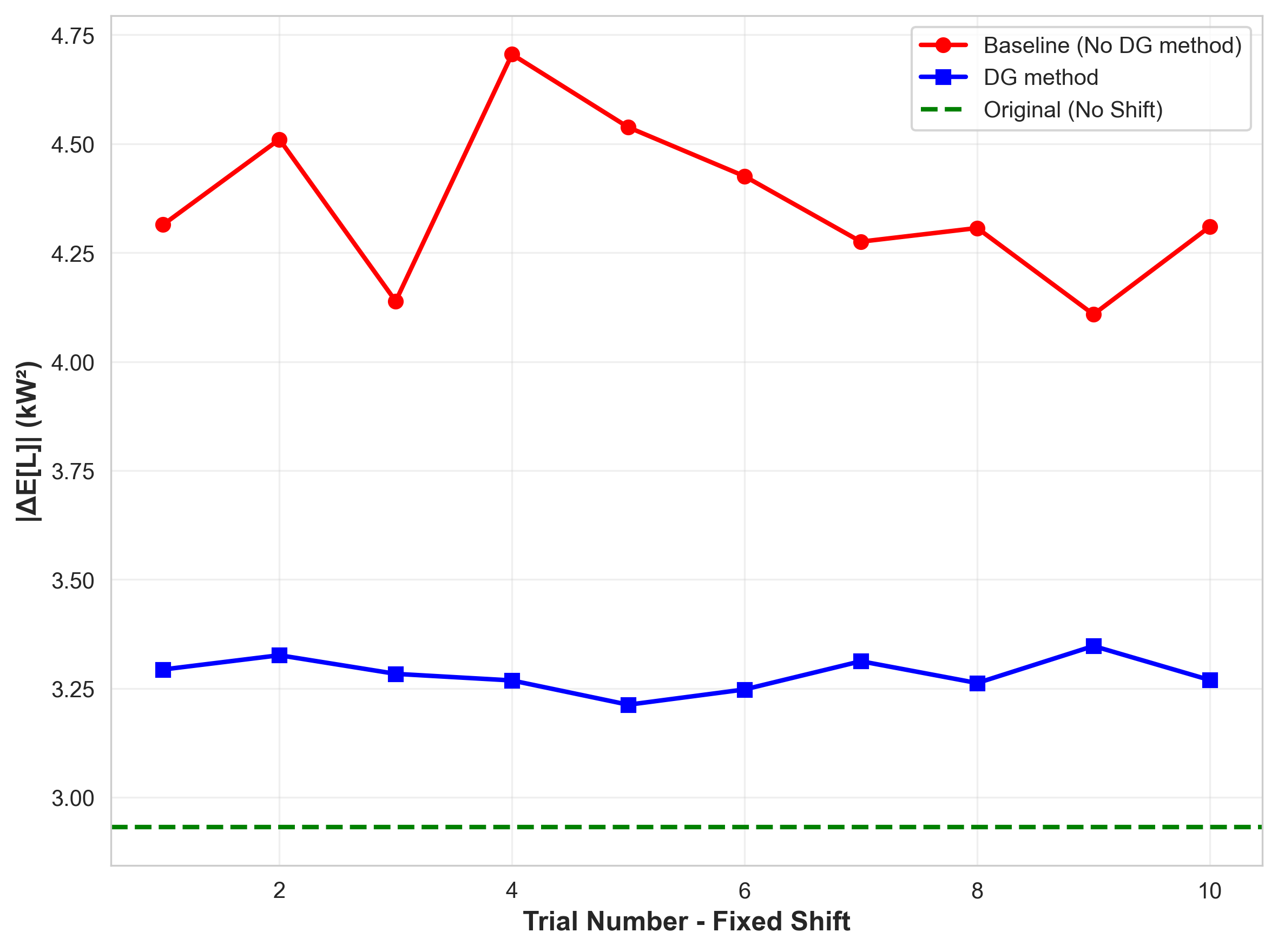}%
    \label{fig:evolution-fixed}}
  \hfil
  \subfloat[MSE evolution—uniform shift]{%
    \includegraphics[width=0.48\textwidth]{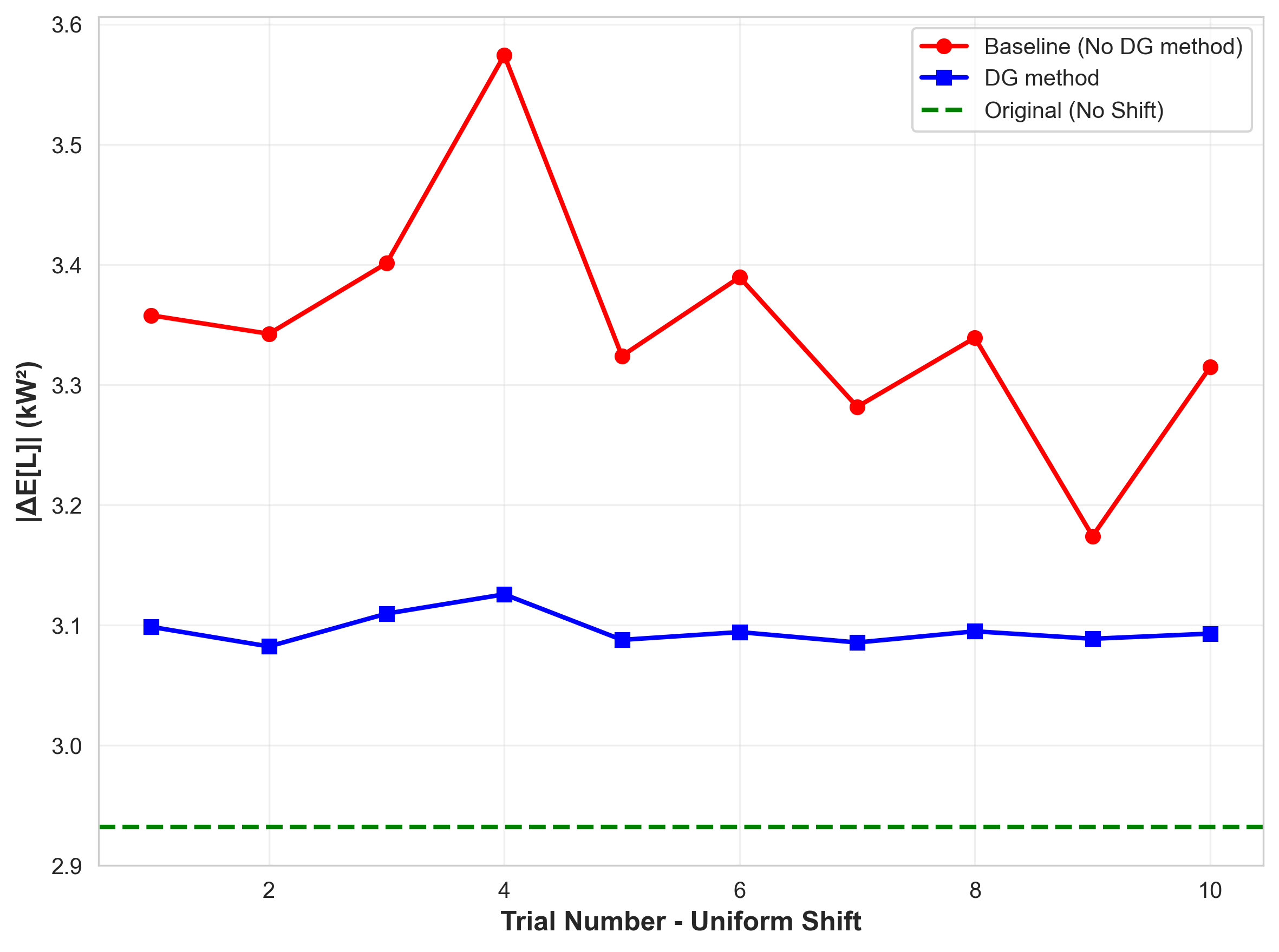}%
    \label{fig:evolution-uniform}}
  \caption{MSE evolution (kW$^2$) across trials under (a) fixed shift, and (b) uniform shift}
  \label{fig:evolution}
\end{figure*}

\bibliographystyle{IEEEtran}
\bibliography{ref}
\end{document}